%% file: output.tex
\def\isarxiv{1}

\ifdefined\isarxiv
\documentclass[12pt]{article}
\usepackage[margin=1.5in]{geometry}
\usepackage[T1]{fontenc}
\usepackage{times}
\else
\documentclass{article}
\usepackage{hyperref}
\usepackage[final]{neurips_2021}
\fi

\usepackage[utf8]{inputenc} 
\usepackage{amsmath}
\usepackage{amssymb}
\usepackage{amsthm}
\usepackage[T1]{fontenc}    
\usepackage{hyperref}       
\usepackage{url}            
\usepackage{booktabs}       
\usepackage{amsfonts}       
\usepackage{nicefrac}       
\usepackage{microtype}      
\usepackage{xcolor}         

\usepackage{amsmath,amsthm,amssymb}
\usepackage{graphicx}
\usepackage{cleveref}
\usepackage{subcaption}
\usepackage{cleveref}
\usepackage{comment}
\usepackage{bbm}
\usepackage{textcomp}
\usepackage{wrapfig}
\usepackage{float}

\usepackage[utf8]{inputenc} 
\usepackage[T1]{fontenc}    
\usepackage{url}            
\usepackage{booktabs}       
\usepackage{amsfonts}       
\usepackage{nicefrac}       
\usepackage{microtype}      
\usepackage{multirow}

\usepackage{amsmath}
\usepackage{amsthm}
\usepackage{amssymb}
\usepackage{algorithm}
\usepackage{color}
\usepackage[english]{babel}
\usepackage{graphicx}
\usepackage{grffile}
\usepackage{wrapfig,epsfig}
\usepackage{epstopdf}
\usepackage{url}
\usepackage{color}
\usepackage{epstopdf}
\usepackage{algpseudocode}
\usepackage[T1]{fontenc}
\usepackage{bbm}
\usepackage{comment}
\usepackage{caption}
\usepackage{dsfont}
\usepackage{tikz}

\usepackage{pgfplots}
\pgfplotsset{compat=1.8}
\tikzset{elegant/.style={smooth,thick,samples=500,magenta}}

\theoremstyle{plain}
\newtheorem{theorem}{Theorem}[section]
\newtheorem{lemma}[theorem]{Lemma}
\newtheorem{remark}[theorem]{Remark}
\newtheorem{corollary}[theorem]{Corollary}
\newtheorem{proposition}[theorem]{Proposition}
\theoremstyle{definition}
\newtheorem{definition}[theorem]{Definition}

\newtheorem{assumption}[theorem]{Assumption}

\newtheorem{example}[theorem]{Example}

\newtheorem{claim}[theorem]{Claim}
\crefname{assumption}{Assumption}{Assumptions}

\input{math_commands}

\input{symbols}

\newcommand{\red}[1]{{\color{red}#1}}

\definecolor{b2}{RGB}{51,153,255}
\definecolor{mygreen}{RGB}{80,180,0}

\title{Going Beyond Linear RL: Sample Efficient Neural Function Approximation}

\ifdefined\isarxiv
\author{
Baihe Huang\thanks{\texttt{baihehuang@pku.edu.cn}. Peking University.}
\and
Kaixuan Huang\thanks{\texttt{kaixuanh@princeton.edu}. Princeton University.}
\and
Sham M. Kakade\thanks{\texttt{sham@cs.washington.edu}. University of Washington}
\and 
Jason D. Lee\thanks{\texttt{Jasondl@princeton.edu}. Princeton University.}
\and 
Qi Lei\thanks{\texttt{qilei@princeton.edu}. Princeton University}
\and
Runzhe Wang\thanks{\texttt{runzhew@princeton.edu}. Princeton University}
\and
Jiaqi Yang\thanks{\texttt{yangjq17@gmail.com}. Tsinghua University}
}
\date{}
\else
\author{
	Baihe Huang\textsuperscript{1$\ast$}, Kaixuan Huang\textsuperscript{2$\ast$}, Sham M. Kakade\textsuperscript{3,4$\ast$}, Jason D. Lee\textsuperscript{2$\ast$}\\ \textbf{ Qi Lei\textsuperscript{2$\ast$}, Runzhe Wang\textsuperscript{2$\ast$}, Jiaqi Yang\textsuperscript{5}}\thanks{Alphabetical order. Correspondence to: Baihe Huang, \url{baihehuang@pku.edu.cn},
	Kaixuan Huang, \url{kaixuanh@princeton.edu},
	Sham M. Kakade, \url{sham@cs.washington.edu},
	Jason D. Lee, \url{Jasondl@princeton.edu},
	Qi Lei, \url{qilei@princeton.edu},
	Runzhe Wang, \url{runzhew@princeton.edu}, and Jiaqi Yang, \url{yangjq17@gmail.com}.}\\
	\\
	\textsuperscript{1}Peking University \quad \textsuperscript{2}Princeton University \quad \textsuperscript{3}Harvard University\\
	\textsuperscript{4}Microsoft Research \quad \textsuperscript{5}{Tsinghua University}
}
\date{}
\fi

\ifdefined\usebigfont

\usepackage{times}
\usepackage[fontsize=13pt]{scrextend}
\AtBeginDocument{
	\newgeometry{left=1.56in,right=1.56in,top=1.71in,bottom=1.77in}
}
\else
\fi
\begin{document}

\ifdefined\isarxiv

\else

\fi

\ifdefined\isarxiv
  \maketitle
  \begin{abstract}
  \input{abstract}
  \end{abstract}


\else
\maketitle
\begin{abstract}
\input{abstract}
\end{abstract}
\fi

\input{sections/intro}

\input{sections/related}

\input{sections/prelim}

\input{sections/nn_simulator}

\input{sections/poly}
\input{sections/sec_lb_det_mdp} 

\section{Conclusions}
\input{sections/conclusion}

\input{acknowledgement} 

\bibliography{simonduref,beyondntk,jason_bib}
\bibliographystyle{alpha}

\newpage
\appendix

\input{appendix/nn_recovery}

\input{appendix/mom}

\input{appendix/gd+mom}

\input{appendix/ub_det_mdp} 
\input{appendix/lb_det_mdp} 

\input{appendix/technical_claims}

\end{document}

%% file: math_commands.tex
\usepackage{amsmath,amsfonts,bm}

\def\1{\bm{1}}

\makeatletter
\newcommand{\ve}{\@ifnextchar\bgroup{\velong}{{\bm{e}}}}
\newcommand{\velong}[1]{{\bm{#1}}}
\makeatother

\DeclareMathAlphabet{\mathsfit}{\encodingdefault}{\sfdefault}{m}{sl}
\SetMathAlphabet{\mathsfit}{bold}{\encodingdefault}{\sfdefault}{bx}{n}

\def\gA{{\mathcal{A}}}

\def\gF{{\mathcal{F}}}

\def\gN{{\mathcal{N}}}

\def\gS{{\mathcal{S}}}

\def\gV{{\mathcal{V}}}

\newcommand{\E}{\mathbb{E}}
\newcommand{\R}{\mathbb{R}}

\DeclareMathOperator*{\argmax}{arg\,max}
\DeclareMathOperator*{\argmin}{arg\,min}

\DeclareMathOperator{\sign}{sign}

\newcommand{\wh}{\widehat}
\newcommand{\wt}{\widetilde}
\newcommand{\ov}{\overline}
\newcommand{\PP}{\mathbb{P}}

\renewcommand{\tilde}{\wt}
\renewcommand{\hat}{\wh}

\newcommand{\T}{{\cal T}}
\newcommand{\F}{{\cal F}}

\newcommand{\states}{\mathcal{S}}

\newcommand{\actions}{\mathcal{A}}

\newcommand{\mdps}{\mathcal{M}}

%% file: symbols.tex
\newcommand{\poly}{\mathrm{poly}}
\newcommand{\polylog}{\mathrm{polylog}}


%% file: abstract.tex
Deep Reinforcement Learning (RL) powered by neural net approximation of the Q function has had enormous empirical success. While the theory of RL has traditionally focused on linear function approximation (or eluder dimension) approaches, little is known about nonlinear RL with neural net approximations of the Q functions. This is the focus of this work, where we study function approximation with two-layer neural networks (considering both ReLU and polynomial activation functions).  Our first result is a computationally and statistically efficient algorithm in the generative model setting under completeness for two-layer neural networks. Our second result considers this setting but under only realizability of the neural net function class.  Here, assuming deterministic dynamics, the sample complexity scales linearly in the algebraic dimension. In all cases, our results significantly improve upon what can be attained with linear (or eluder dimension) methods.

%% file: sections/intro.tex
\section{Introduction}

In reinforcement learning (RL), an agent aims to learn the optimal decision-making rule by interacting with an unknown environment \cite{sutton2018reinforcement}. Deep Reinforcement Learning, empowered by deep neural networks \cite{lecun2015deep, goodfellow2016deep}, has achieved tremendous success in various real-world applications, such as Go \cite{silver2016alphago}, Atari \cite{mks+13}, Dota2 \cite{bbc+19}, Texas Hold\'em poker \cite{moravvcik2017deepstack}, and autonomous driving \cite{shalev2016safe_drive}. Those modern RL applications are characterized by large state-action spaces, and the empirical success of deep RL corroborates the observation that deep neural networks can extrapolate well across state-action spaces~\cite{henderson2018deep,mnih2013playing,lillicrap2015continuous}.

Although in practice non-linear function approximation scheme is prevalent, theoretical understandings of the sample complexity of RL mainly focus on tabular or linear function approximation settings~\cite{strehl2006pac,jaksch2010near,azar2017minimax,jin2018q,russo2019worst,zanette2019tighter,abbasi2019exploration,jin2019linear,jin2019provably, wang2019optimism}. These results rely on finite state space or exact linear approximations. Recently, sample efficient algorithms under non-linear function approximation settings are proposed \cite{wen2017efficient,dann2018oracle,du2019provably,dong2019sqrt,lcyw19,wcyw20,dym21,zhou2020neuralucb,yang2020kernel}. Those algorithms are based on Bellman rank \cite{jiang2017contextual}, eluder dimension \cite{rvr13}, neural tangent kernel \cite{jgh18,als19_dnn,dllwz19,zou2018stochastic}, or sequential Rademacher complexity \cite{rst15a,rst15b}. However, the understanding of how deep RL learns and generalizes in large state spaces is far from complete. Whereas the aforementioned works study function approximation structures that possess the nice properties of linear models, such as low information gain and low eluder dimensions, the highly non-linear nature of neural networks renders challenges on the applicability of existent analysis to deep RL. For one thing, recent wisdoms in deep learning theory cast doubt on the ability of neural tangent kernel and random features to model the actual neural networks. Indeed, the neural tangent kernel approximately reduces neural networks to random feature models, but the RKHS norm of neural networks is exponential \cite{yehudai2019power}. Moreover, it remains unclear what neural networks models have low eluder dimensions. For example, recent works \cite{dym21,li21eluder} show that two layer neural networks have exponential eluder dimension in the dimension of features. Furthermore, \cite{malik21sample} demonstrates hard instances where learning RL is exponentially hard even in the case that the target value function can be approximated as a neural network and the optimal policy is softmax linear. Thus, the mismatch between the empirical success of deep RL and its theoretical understanding remains significant, which yields the following important question:
\begin{center}
	\textit{What are the structural properties that allow sample-efficient algorithms for RL with neural network function approximation?}
\end{center}

In this paper, we advance the understanding of the above question by displaying several structural properties that allow efficient RL algorithms with neural function approximations. We consider several value function approximation models that possess high information gain and high eluder dimension. Specifically, we study two structures, namely two-layer neural networks and structured polynomials (i.e. two-layer neural networks with polynomial activation functions), under two RL settings, namely RL with simulator model and online RL. In the simulator (generative model) setting \cite{kak03,sww+18}, the agent can simulate the MDP at any state-action pair. In online RL, the agent can only start at an initial state and interact with the MDP step by step. 
The goal in both settings is to find a near-optimal policy while minimizing the number of samples used. 

We obtain the following results. For the simulator setting, we propose sample-efficient algorithms for RL with two-layer neural network function approximation. Under either policy completeness, Bellman completeness, or gap conditions, our method provably learns near-optimal policy with polynomial sample complexities. For online RL, we provide sample-efficient algorithms for RL with structured polynomial function approximation. When the transition is deterministic, we also present sample-efficient algorithms under only the realizability assumption \cite{du2019representation,wwk21}.
Our main techniques are based on neural network recovery \cite{zsj+17,jsa15,glm18}, and algebraic geometry \cite{milneAG,shafarevich2013basic,bochnak2013real,wang2019generalized}.

\subsection{Summary of our results}

Our main results in different settings are summarized in Table~\ref{tab:results}. We consider two-layer neural networks $f(x) = \langle v, \sigma(Wx) \rangle$ (where $\sigma$ is ReLU activation) and rank $k$  polynomials (see Example \ref{ex:rank-k}).
\begin{table}[h]
	\centering
	\caption[Caption for Table]{Baselines and our main results for the sample complexity to find an $\epsilon$-optimal policy.}
	\resizebox{\textwidth}{!}{
		\begin{tabular}{ |c|c|c|c|c|c| }
			\hline
			\multirow{2}{*}{} & \multicolumn{2}{c|}{rank $k$ polynomial} & \multicolumn{3}{c|}{ Neural Net of Width $k$} \\
			\cline{2-6}
			& Sim. + Det. (R) & Onl. + Det. (R)& Sim. + Det. (R) & Sim. + Gap. (R) & Sim. + Stoch. (C)\\\hline
			Baseline & $O(d^p)$ &  $O(d^p)$ & $O(d^{\poly(1/\epsilon)})$ (*) & $O(d^{\poly(1/\epsilon)})$ & $O(d^{\poly(1/\epsilon)})$ \\
			\hline
			Our results & $\red{O(dk)}$ & $\red{O(dk)}$ & $\red{\Tilde{O}(\poly(d)\cdot \exp(k))}$ & $\red{\Tilde{O}( \poly(d,k))}$ & $\red{\Tilde{O}( \poly(d,k)/\epsilon^2)}$ \\
			\hline
	\end{tabular}}
	\label{tab:results}
\end{table}

We make the following elaborations on Table~\ref{tab:results}.
\begin{itemize}
	\item {For simplicity, we display only the dependence on the feature dimension $d$, network width or polynomial rank $k$, precision $\epsilon$, and degree $p$ (of polynomials)}. 
	\item In the table \textit{Sim.} denotes simulator model, \textit{Onl.} denotes online RL, \textit{ Det.} denotes deterministic transitions, \textit{ Stoch.} denotes stochastic transitions, \textit{ Gap.} denotes gap condition, \textit{(R)} denotes realizability assumption only, and \textit{(C)} denotes completeness assumption (either policy complete or Bellman complete) together with realizability assumption. 
	\item We apply \cite{du2020agnostic} for the deterministic transition baseline, and apply \cite{du2019good} for the stochastic transition baseline. We are unaware of any methods that can directly learn  MDP with neural network value function approximation\footnotemark. 
	\item In polynomial case, the baseline first vectorizes the tensor $\begin{pmatrix}1 \\x\end{pmatrix}^{\otimes p}$ into a vector in $\R^{(d+1)^p}$ and then performs on this vector.  In the neural network case, the baseline uses a polynomial of degree $1/\epsilon$ to approximate the neural network with precision $\epsilon$ and then vectorizes the polynomial into a vector in $\R^{d^{\poly(1/\epsilon)}}$. The baseline method for realizable model (denoted by (*)) needs a further gap assumption of $\text{gap} \ge d^{\poly(1/\epsilon)} \epsilon$ to avoid the approximation error from escalating~\cite{du2020agnostic}; note for small $\epsilon$ this condition never holds but we include it in the table for the sake of comparison.
	\item In rank $k$ polynomial case, our result $O(dk)$ in simulator model can be found in Theorem~\ref{thm:gen-rl} and our result $O(dk)$ in online RL model can be found in Theorem~\ref{thm:online-rl}. These results only require a realizability assumption. Efficient explorations are guaranteed by algebraic-geometric arguments. In neural network model, our result $\Tilde{O}(\poly(d) \cdot \exp(k))$ in simulator model can be found in Theorem~\ref{thm:det_nn}. This result also only relies on the realizability assumption. For stochastic transitions, our result $\Tilde{O}(\poly(d,k)/\epsilon^2)$ works for either policy complete or Bellman complete settings, as in Theorem~\ref{thm:policy_nn} and Theorem~\ref{thm:bellman_nn} respectively. The $\Tilde{O}(\poly(d,k))$ result for gap condition can be found in Theorem~\ref{thm:gap_Q_star}.
\end{itemize}
\footnotetext{Prior work on neural function approximation has focused on neural tangent kernels, which would require $d^{\poly(1/\epsilon)}$ to approximate a two-layer network~\cite{ghorbani2021linearized}. }

%% file: sections/related.tex
\subsection{Related Work}

\paragraph{Linear Function Approximation.}  RL with linear function approximation has been widely studied under various settings, including linear MDP and linear mixture MDP \cite{jin2019provably,zanette2020learning,yang2019reinforcement}. While these papers have proved efficient regret and sample complexity bounds, their analyses relied heavily on two techniques: they used the confidence ellipsoid to quantify the uncertainty, and they used the elliptical potential lemma to bound the total uncertainty \cite{abbasi2011improved}. These two techniques were integral to their analyses but are so restrictive that they generally do not extend to nonlinear cases.

\paragraph{Eluder Dimension.} \cite{russo2013eluder,osband2013more} proposed eluder dimension, a complexity measure of the function space, and proved regret and sample complexity bounds that scaled with the eluder dimension, for bandits and reinforcement learning \cite{wsy20,jlm21}. They also showed that the eluder dimension is small in several settings, including generalized linear models and LQR.
However, as shown in \cite{dym21}, the eluder dimension could be exponentially large even with a single ReLU neuron, which suggested the eluder dimension would face difficulty in dealing with neural network cases. 
The eluder dimension is only known to give non-trivial bounds for linear function classes and monotone functions of linear function classes. For structured polynomial classes, the eluder dimension simply embeds into an ambient linear space of dimension $d^p$, where $d$ is the dimension, and $p$ is the degree. This parallels the lower bounds in  linearization / neural tangent kernel (NTK) works \cite{wei2019regularization,ghorbani2019linearized,allen2019can}, which show that linearization also incurs a similarly large penalty of $d^p$ sample complexity, and more advanced algorithm design is need to circumvent linearization\cite{bai2019beyond,chen2020towards,fang2020modeling,woodworth2019kernel,gao2019convergence,nacson2019lexicographic,ge2018learning,moroshko2020implicit,haochen2020shape,wang2020beyond,damian2021label}.

\paragraph{Bellman Rank and Completeness.} \cite{jiang2017contextual,sun2019model} studied RL with general function approximation. They used Bellman rank to measure the error of the function class under the Bellman operator and gave proved bounds in the term of it. Recently, \cite{dkl+21} propose bilinear rank and encompass more function approximation models. However, it is hard to bound either the Bellman rank or the bilinear rank for neural nets. Therefore, their results are not known to include the neural network approximation setting. Another line of work shows that exponential sample complexity is unavoidable even with good representations \cite{du2019good,weisz2020exponential,wwk21}, which implies the realizability assumption alone might be insufficient for function approximations.

\paragraph{Deterministic RL}

Deterministic system is often the starting case in the study of sample-efficient algorithms, where the issue of exploration and exploitation trade-off is more clearly revealed since both the transition kernel and reward function are deterministic. The seminal work \cite{wen2013efficient} proposes a sample-efficient algorithm for Q-learning that works for a family of function classes. Recently, \cite{du2020agnostic} studies the agnostic setting where the optimal Q-function can only be approximated by a function class with approximation error. The algorithm in \cite{du2020agnostic} learns the optimal policy with the number of trajectories linear with the eluder dimension.

%% file: sections/prelim.tex
\section{Preliminaries}

An episodic Markov Decision Process (MDP) is defined by the tuple $\mdps = (\states,\actions,H,\PP,r)$ where $\states$ is the state space, $\actions$ is the action set, $H$ is the number of time steps in each episode, $\PP$ is the transition kernel and $r$ is the reward function. In each episode the agent starts at a fixed initial state $s_1$ and at each time step $h \in [H]$ it takes action $a_h$, receives reward $r_h(s_h,a_h)$ and transits to $s_{h+1} \sim \PP(\cdot|s_h,a_h)$.

A deterministic policy $\pi$ is a length-$H$ sequence of functions $\pi = \{\pi_h : \states \mapsto \actions\}_{h=1}^H$. Given a policy $\pi$, we define the value function $V_h^\pi(s)$ as the expected sum of reward under policy $\pi$ starting from $s_h = s$:
\begin{align*}
    V_h^\pi(s) := \E\left[\sum_{t = h}^H r_t(s_t,a_t)|s_h = s\right]
\end{align*}
and we define the Q function $Q_h^\pi(s,a)$ as the the expected sum of reward taking action $a$ in state $s_h = s$ and then following $\pi$:
\begin{align*}
    Q_h^\pi(s,a):= \E\left[\sum_{t = h}^H r_t(s_t,a_t)|s_h = s,a_h = a\right].
\end{align*}
The Bellman operator $\T_h$ applied to Q-function $Q_{h+1}$ is defined as follow
\begin{align*}
    \T_h(Q_{h+1})(s,a) := r_h(s,a)  + \E_{s' \sim \PP(\cdot|s,a)}[\max_{a'} Q_{h+1}(s',a')].
\end{align*}
There exists an optimal policy $\pi^\ast$ that gives the optimal value function for all states, i.e. $V^{\pi^\ast}_h(s) = \sup_{\pi}V_h^\pi(s)$ for all $h \in [H]$ and $s \in \states$. For notational simplicity we abbreviate $V^{\pi^\ast}$ as $V^{\ast}$ and correspondingly $Q^{\pi^\ast}$ as $Q^\ast$. Therefore $Q^\ast$ satisfies the following Bellman optimality equations for all $s \in \states$, $a \in \actions$ and $h \in [H]$:
\begin{align*}
Q^\ast_h(s,a) = \T_h(Q^\ast_{h+1})(s,a).
\end{align*}

The goal is to find a policy $\pi$ that is $\epsilon$-optimal in the sense that $V_1^\ast(s_1) - V_1^\pi(s_1) \leq \epsilon$, within a small number of samples. We consider two query models of interacting with MDP:
\begin{itemize}
    \item In the simulator model (\cite{kak03}, \cite{sww+18}), the agent interacts with a black-box that simulates the MDP. At each time step $h \in [H]$, the agent can start at a state-action pair $(s,a)$ and interact with the black box by executing some policy $\pi$ chosen by the agent.
    \item In online RL, the agent can only start at the initial state and interact with the MDP by using a policy and observing the rewards and the next states. In each episode $k$, the agent proposes a policy $\pi^k$ based on all history up to episode $k-1$ and executes $\pi^k$ to generate a single trajectory $\{s_h^k,a_h^k\}_{h=1}^H$ with $a_h^k = \pi^k_h(s_h^k)$ and $s_{h+1}^k \sim \PP(\cdot |s_h^k,a_h^k)$.
\end{itemize}

\subsection{Function approximation}

In reinforcement learning with value function approximation, the learner is given a function class $\F = \F_1 \times \cdots \times \F_H$, where $\F_h \subset \{ f: \states \times \actions \mapsto [0,1]\}$ is a set of candidate functions to approximate $Q^\ast$. The following assumption is commonly adopted in the literature \cite{jin2019linear,wang2020provably,jlm21,du2021bilinear}.
\begin{assumption}[Realizability]\label{asp:realizability}
$Q^\ast_h \in \F_h$ for all $h \in [H]$.
\end{assumption}
The function approximation is equipped with feature mapping $\phi: \states \times \actions \mapsto \{u \in \R^d: \|u\|_2 \leq B_\phi\}$ that is known to the agent. We focus the continuous action setting (e.g. in control and robotics problems) and make the following regularity assumption on the feature function $\phi$.
\begin{assumption}[Bounded features]\label{asp:dense_feature}
Assume $\phi(s,a) \leq B_{\phi}, \forall (s,a) \in {\cal S \times A}$.
\end{assumption}

\paragraph{Notation}

For any vector $x \in \R^d$, let $x_{\max} := \max_{i \in [d]} x_i$ and $x_{\min}:= \min_{i \in [d]} x_i$. Let $s_i(\cdot)$ denote the $i$-th singular value, $s_{\min}(\cdot)$ denotes the minimum eigenvalue and $s_{\max}(\cdot)$ denotes the maximum eigenvalue. The conditional number is defined by $\kappa(\cdot) := s_{\max}(\cdot)/s_{\min}(\cdot)$. We use $\otimes$ to denote Kronecker product and $\circ$ to denote Hadamard product. For a given integer $H$, we use $[H]$ to denote the set $\{1, 2, \ldots, H\}$. For a function $f: \mathfrak{X} \mapsto \mathfrak{Y}$, we use $f^{-1}(y) := \{x \in \mathfrak{X}: f(x) = y \}$ to denote the preimage of $y \in \mathfrak{Y}$. We use the shorthand $x \lesssim y$ ($x \gtrsim y$) to indicate $x \leq O(y)$ ($x \geq \Omega(y)$).

%% file: sections/nn_simulator.tex
\section{Neural Network Function Approximation}\label{sec:nn_simulator}

In this section we show sample-efficient algorithms with neural network function approximations. The function class of interest is given in the following definition. More general neural network class is discussed in Appendix~\ref{sec:nn_recovery_guarantee}.

\begin{definition}[Neural Network Function Class]\label{def:nn_class}
	We use $\F_{NN}$ to denote the function class of $f(\phi(s,a)) : \states \times \actions \mapsto \R$ where $f(x) = \langle v, \sigma(W x) \rangle: \mathbb{R}^d \mapsto \R$ is a two-layer neural network where $\sigma$ is ReLU, $\|W\|_F \leq B_W$, $v \in \{\pm 1\}^k$, $\prod_{i=1}^k s_i(W)/s_{\min}(W) \leq \lambda$, $s_{\max}(W)/s_{\min}(W) \leq \kappa$ and $k \le d$. Here $\phi : \mathcal{A} \times \mathcal{S} \mapsto \mathbb{R}^d$ is a known feature map whose image contains a ball $\{u \in \R^d: \|u\|_2 \leq  \delta_\phi \}$ with $\delta_\phi \geq d\cdot\polylog(d)$.\footnote{Here the $\delta_\phi$ is chosen only for simplicity. In general this assumption can be relaxed to that the image of $\phi$ contains an arbitrary dense ball near the origin, since one can always rescale the feature mapping in the neural function approximation.}  
\end{definition}

We introduce the following completeness properties in the setting of value function approximations. Along with Assumption~\ref{asp:realizability}, they are commonly adopted in the literature .

\begin{definition}[Policy complete]
	Given MDP $\mdps = (\states,\actions,\PP,r,H)$, function class $\F_h: \states \times \actions \mapsto \R ,h \in [H]$ is called policy complete if for all $\pi$ and $h \in [H]$, $Q^{\pi}_h \in \F_h$.
\end{definition}

\begin{definition}[Bellman complete]
	Given MDP $\mdps = (\states,\actions,\PP,r,H)$, function class $\F_h: \states \times \actions \mapsto \R ,h \in [H]$ is called Bellman complete if for all $h \in [H]$ and $Q_{h+1} \in \F_{h+1}$, $\T_h(Q_{h+1}) \in \F_h$.
\end{definition}

\subsection{Warmup: Realizable $Q^\ast$ with deterministic transition}
\label{sec:det_nn}
We start by considering the case when the transition kernel is deterministic. In this case only Assumption~\ref{asp:realizability} is required for the expressiveness of neural network function approximations. Algorithm~\ref{alg:det_nn} learns optimal policy from time step $H$ to $1$. Suppose we have learned policies $\pi_{h+1} ,\dots,\pi_H$ at level $h$ and they are exactly the optimal policies. We first explore features $\phi(s_h^i,a_h^i)$ over a standard Gaussian distribution, and if $\|\phi(s_h^i,a_h^i)\|_2 \geq \delta_\phi$ then we simply skip this trial. Recall that $\delta_\phi \geq d \cdot \poly \log (d)$, so with high probability ({w.r.t $d$}) almost all feature samples will be explored. We next construct an estimate $\hat Q_h^i$ of $Q^\ast(s^i_h,a^i_h)$ by collecting cumulative rewards using $\pi_{h+1}, \dots,\pi_{H}$ as the roll-out. Since the transition is deterministic, $\hat Q_h^i=Q^\ast(s^i_h,a^i_h)$ for all explored samples $(s^i_h,a^i_h)$. Recall that $Q^\ast_h$ is a two-layer neural network, we can now recover its parameters in Line~\ref{stp:recover_det_nn} exactly by invoking techniques in neural network optimization (see, e.g. \cite{jsa15,glm17,zsj+17}). Details of this step can be found in Appendix~\ref{sec:nn_recovery_guarantee}, where the method is mainly based on \cite{zsj+17}. This means the reconstructed $\hat Q_h$ in Line~\ref{lin:reconstructed_Q_det} is precisely $Q^\ast$, and the algorithm can thus find optimal policy $\pi^\ast_h$ in the $h$-th level.

\begin{algorithm*}[h]\caption{Learning realizable $Q^\ast$ with deterministic transition}\label{alg:det_nn}
	\begin{algorithmic}[1]
		\For{$h=H,\ldots 1$}
		\State Sample $x^i_h, i \in [n]$ from standard Gaussian $\mathcal{N}(0,I_d)$
		\For{$i \in [n]$}
		\If{$\|x^i_h\| \leq \delta_\phi$}
		\State Find $(s^i_h,a^i_h) \in \phi^{-1}(x^i_h)$ and locate the state $s^i_h$ in the generative model
		\State Pull action $a^i_h$ and use $\pi_{h+1}, \dots,\pi_{H}$ as the roll-out to collect rewards $r_{h}^{(i)} , \dots, r_{H}^{(i)}$
		\State Construct estimation $${\hat Q}_h^i \leftarrow r_{h}^{(i)} + \cdots + r_{H}^{(i)}$$
		\Else
		\State Let ${\hat Q}_h^i \leftarrow 0$.
		\EndIf
		\EndFor
		\State Compute $(v_h,W_h) \leftarrow \textsc{NeuralNetRecovery}(\{(x^i_h,{\hat Q}_h^i): i \in [n] \})$ \label{stp:recover_det_nn}
		\State Set $\hat Q_h(s,a) \leftarrow  v_h^\top \sigma(W_h \phi(s,a))$ \label{lin:reconstructed_Q_det}
		\State Let $\pi_h(s) \leftarrow \argmax_{a \in \states}~ \hat Q_h(s,a)$
		\EndFor
		\State {\bf Return} $\pi_1,\dots,\pi_H$
	\end{algorithmic}
\end{algorithm*}

\begin{theorem}\label{thm:det_nn}
(Informal) If $n \geq d \cdot \poly(\kappa,k,\lambda,\log d, B_W, B_{\phi}, H)$, then with high probability Algorithm~\ref{alg:det_nn} learns the optimal policy.
\end{theorem}
The formal statement and complete proof are deferred to the Appendix~\ref{sec:proof_det_nn}. The main idea of exact neural network recovery can be summarized in the following. We first use method of moments to find a `rough' parameter recovery. If this `rough' recovery is sufficiently close to the true parameter, the empirical $\ell_2$ loss function is locally strongly convex and there is unique global minimum. Then we can apply gradient descent to find this global minimum which is exactly the true parameter.
\subsection{Policy complete neural function approximation}\label{sec:policy_nn}

Now we consider general stochastic transitions. Difficulties arise in this scenario due to noises in the estimation of Q-functions. In the presence of model misspecification, these noises cause estimation errors to amplify through levels and require samples to be exponential in $H$. In this section, we show that neural network function approximation is still learnable, assuming the function class $\F_{NN}$ is policy complete with regard to MDP $\mdps$. Thus for all $\pi \in \Pi$, we can denote $Q_h^{\pi}(s,a) = \langle v^\pi, \sigma(W^\pi \phi(s,a)) \rangle$.

\begin{algorithm*}[h]\caption{Learn policy complete NN with simulator.}\label{alg:simulator_nn_policy_complete}
	\begin{algorithmic}[1]
		\For{$h=H,\ldots 1$}
		\State Sample $x^i_h, i \in [n]$ from standard Gaussian $\mathcal{N}(0,I_d)$ 
		\For{$i \in [n]$}
		\If{$\|x^i_h\| \leq \delta_\phi$}
		\State Find $(s^i_h,a^i_h) \in \phi^{-1}(x^i_h)$ and locate the state $s^i_h$ in the generative model
		\State Pull action $a^i_h$ and use $\pi_{h+1}, \dots,\pi_{H}$ as the roll-out to collect rewards $r_{h}^{(i)} , \dots, r_{H}^{(i)}$
		\State Construct unbiased estimation of ${ Q}_h^{\pi_{h+1}, \dots,\pi_{H}}(s^i_h,a^i_h)$
		$${\hat Q}_h^i \leftarrow r_{h}^{(i)} + \cdots + r_{H}^{(i)}$$
		\Else
		\State Let ${\hat Q}_h^i \leftarrow 0$.
		\EndIf
		\EndFor
		\State Compute $(v_h,W_h) \leftarrow \textsc{NeuralNetNoisyRecovery}(\{(x^i_h,{\hat Q}_h^i): i \in [n] \})$
		\State Set $\hat Q_h(s,a) \leftarrow  v_h^\top \sigma(W_h \phi(s,a))$
		\State Let $\pi_h(s) \leftarrow \argmax_{a \in \states}~ \hat Q_h(s,a)$
		\EndFor
		\State {\bf Return} $\pi_1,\dots,\pi_H$
	\end{algorithmic}
\end{algorithm*}

Algorithm~\ref{alg:simulator_nn_policy_complete} learns policy from level $H, H-1, \dots, 1$. In level $h$, the algorithm has learned policy $\pi_{h+1}, \dots,\pi_{H}$ that is only sub-optimal by $(H-h)\epsilon/H$. Then it explores features $\phi(s,a)$ from $\mathcal{N}(0,I_d)$. The algorithm then queries $(s,a)$ and uses learned policy $\pi_{h+1}, \dots,\pi_{H}$ as roll out to collect an unbiased estimate of the Q-function ${Q}_h^{\pi_{h+1}, \dots,\pi_{H}}(s,a)$. Since ${Q}_h^{\pi_{h+1}, \dots,\pi_{H}}(s,a) \in \F_{NN}$ is a two-layer neural network, it can then be recovered from samples. Details of this step can be found in Appendix~\ref{sec:nn_recovery_guarantee}, where the methods are mainly based on \cite{zsj+17}. The algorithm then reconstructs this Q-function and finds its optimal policy $\pi_h$.

\begin{theorem}\label{thm:policy_nn}
(Informal) Fix $\epsilon,t$, if $n \geq \epsilon^{-2} \cdot d \cdot \poly(\kappa,k, B_W, B_{\phi}, H, \log (d/t))$, then with probability at least $1-t$ Algorithm~\ref{alg:simulator_nn_policy_complete} returns an $\epsilon$-optimal policy $\pi$. 
\end{theorem}
The formal statement and complete proof are deferred to Appendix~\ref{sec:proof_policy_nn}. Notice that unlike the case of Theorem \ref{thm:det_nn}, the sample complexity does not depend on $\lambda$, thus avoiding the potential exponential dependence in $k$.

The main idea of the proof is that at each time step a neural network surrogate of $Q^\ast$ can be constructed by the policy already learned. Suppose we have learned $\pi_{h+1},\dots,\pi_H$ in level $h$, then from policy completeness $Q^{\pi_{h+1},\dots,\pi_H}_h$ belongs to $\F_{NN}$ and we can interact with the simulator to obtain its estimate $\hat Q_h$. If $\|\hat Q_h - Q^{\pi_{h+1},\dots,\pi_H}_h\|_\infty$ is small, the optimistic planning based on $\hat Q_h$ is not far from the optimal policy of $Q^{\pi_{h+1},\dots,\pi_H}_h$. Therefore the errors can be decoupled into the errors in recovering $Q^{\pi_{h+1},\dots,\pi_H}_h$ and the suboptimality of $Q^{\pi_{h+1},\dots,\pi_H}_h$, which depends on level $h+1$. This reasoning can then be recursively performed to level $H$, and thus we can bound the suboptimality of $\pi_h$.

\subsection{Bellman complete neural function approximation}
\label{sec:bellman_nn}

In addition to policy completeness, we show that neural network function approximation can also learn efficiently under the setting where the function class $\F_{NN}$ is Bellman complete with regard to MDP $\mdps$. Specifically, for $Q_{h+1} \in \F_{h+1}$, there are $v^{Q_{h+1}}$ and $W^{Q_{h+1}}$ such that $\T_h(Q_{h+1})(s,a) = \langle v^{Q_{h+1}}, \sigma(W^{Q_{h+1}} \phi(s,a)) \rangle$. 

Algorithm~\ref{alg:simulator_nn_bellman_complete} is similar to the algorithm in previous section. Suppose in level $h$, the algorithm has constructed the Q-function $\hat Q_{h+1}(s,a) = v_{h+1}^\top \sigma(W_{h+1} \phi(s,a))$ that is $(H-h)\epsilon/H$-close to the optimal $Q^\ast_{h+1}$. It then recovers weights $v_{h}, W_{h}$  from $\T_h(\hat Q_{h+1})(s,a) = \langle v^{\hat Q_{h+1}}, \sigma(W^{\hat Q_{h+1}} \phi(s,a)) \rangle$, using unbiased estimates $r_h(s^i_h,a^i_h) + \hat V_{h+1}(s_{h+1}^i)$. The Q-function $\hat Q_{h}(s,a) = v_{h}^\top \sigma(W_{h} \phi(s,a))$ reconstructed from weights $v_{h}, W_{h}$ is thus $(H-h+1)\epsilon/H$-close to the $Q^\ast_h$. 

\begin{algorithm*}[h]\caption{Learn Bellman complete NN with simulator.}\label{alg:simulator_nn_bellman_complete}
	\begin{algorithmic}[1]
		\For{$h=H,\ldots 1$}
		\State Sample $x^i_h, i \in [n]$ from standard Gaussian $\mathcal{N}(0,I_d)$
		\For{$i \in [n]$}
		\If{$\|x^i_h\| \leq \delta_\phi$}
		\State Find $(s^i_h,a^i_h) \in \phi^{-1}(x^i_h)$ and locate the state $s^i_h$ in the generative model
		\State Pull action $a^i_h$ and and observe $r_h(s^i_h,a^i_h), s_{h+1}^i$
		\State Construct unbiased estimation of $\T_h(\hat Q_{h+1})(s^i_h,a^i_h)$ $${\hat Q}_h^i \leftarrow r_h(s^i_h,a^i_h) + \hat V_{h+1}(s_{h+1}^i)$$
		\Else
		\State Let ${\hat Q}_h^i \leftarrow 0$.
		\EndIf
		\EndFor
		\State Compute $(v_h,W_h) \leftarrow \textsc{NeuralNetNoisyRecovery}(\{(x^i_h,{\hat Q}_h^i): i \in [n] \})$
		\State Set $\hat Q_h(s,a) \leftarrow  v_h^\top \sigma(W_h \phi(s,a))$ and $\hat V_h \leftarrow \max_{a \in \actions} \hat Q_h(s,a)$
		\State Let $\pi_h(s) \leftarrow \argmax_{a \in \actions}~ \hat Q_h(s,a)$
		\EndFor
		\State {\bf Return} $\pi_1,\dots,\pi_H$
	\end{algorithmic}
\end{algorithm*}

\begin{theorem}\label{thm:bellman_nn}
(Informal) Fix $\epsilon,t$, if $n \geq \epsilon^{-2} \cdot d \cdot \poly(\kappa,k,B_W, B_{\phi}, H, \log (d/t))$, then with probability at least $1-t$ Algorithm~\ref{alg:simulator_nn_bellman_complete} returns an $\epsilon$-optimal policy $\pi$.  
\end{theorem}

Due to Bellman completeness, the error of estimation $\hat Q_h$ can be controlled recursively. In fact, we can show $\|\hat Q_h - Q^\ast(s,a)\|_\infty$ is small by induction. 
The formal statement and detailed proof are deferred to Appendix~\ref{sec:proof_bellman_nn}. Similar to Theorem \ref{thm:policy_nn}, the sample complexity does not explicitly depend on $\lambda$, thus avoiding potentially exponential dependence in $k$.

\subsection{Realizable $Q^\ast$ with optimality gap}\label{sec:gap_Q_star}

In this section we consider MDPs where there is a non-zero gap between the optimal policy and any other ones. This concept, known as optimality gap, is widely used in reinforcement learning and bandit literature \cite{du2019provably,du2019good,du2020agnostic}.

\begin{definition}\label{def:opt_gap}
The optimality gap is defined as 
\begin{align*}
    \rho = \inf_{a: Q^\ast (s,a) \neq V^\ast(s)} V^\ast(s) - Q^\ast(s,a).
\end{align*}
\end{definition}

We show that in the presence positive optimality gap, there exists an algorithm that can learn the optimal policy with polynomial samples even without the completeness assumptions. Intuitively, this is because one only needs to recover the neural network up to precision $\rho/4$ in order to make sure the greedy policy is identical to the optimal one. The formal statement and proof are deferred to Appendix~\ref{sec:proof_gap_nn}.

\begin{theorem}\label{thm:gap_Q_star}
(Informal) Fix $t \in (0,1)$, if $n = \frac{d}{\rho^2} \cdot \poly(\kappa,k, B_W, B_{\phi}, H, \log(d/t))$, then with probability at least $1-t$ there exists an algorithm that returns the optimal policy $\pi^\ast$.
\end{theorem}

\begin{remark}
In all aforementioned methods, there are two key components that allow efficient learning. First, the exploration is conducted in a way that guarantees an $\ell_\infty$ recovery of candidate functions. By $\ell_\infty$ recovery we mean the algorithm recovers a candidate Q-function in this class deviating from the target function $Q^\ast$ by at most $\epsilon$ uniformly for all state-action pairs in the domain of interest. This notion of learning guarantee has received study in active learning \cite{hanneke2014active,krishnamurthy2017active} and recently gain interest in contextual bandits \cite{dylan2018practical}. Second, the agent constructs unbiased estimators of certain approximations to $Q^\ast$ that lie in the neural function approximation class. This allows the recovery error to decouple linearly across time steps, which is made possible in several well-posed MDP instances, such as deterministic MDPs, MDPs with completeness assumptions, and MDPs with gap conditions. In principle, we note that provably efficient RL algorithms with general function approximation is possible as long as the above two components are  present. We will see in the next section another example of learning RL with highly non-convex function approximations, where the function class of interest, admissible polynomial families, also allows for exploration schemes to achieve $\ell_\infty$ recovery.
\end{remark}

%% file: sections/poly.tex
\section{Polynomial Realizability}
\label{sec:det:mdp}

In this section, we study the sample complexity to learn deterministic MDPs under polynomial realizability. We identify sufficient and necessary conditions for efficiently learning the MDPs for two different settings --- the generative model setting and the online RL setting. Specifically, we show that if the image of the feature map $\phi_h(s_h,a_h)$ satisfies some positive measure conditions, then by random exploring, we can identify the optimal policy with samples linear in the algebraic dimension of the underlying polynomial class. We also provide a lower bound example showing the separation between the two settings.

Next, we introduce the notion of \textbf{Admissible Polynomial Families}, which are the families of structured polynomials that enable efficient learning.

\begin{definition}[Admissible Polynomial Families]
	For $x \in \mathbb{R}^d$, denote $\widetilde{x} = [1, x^\top]^\top $. Let $\mathcal{X}: = \left\{ \widetilde{x} ^{\otimes p}: x \in \mathbb{R}^d\right\}$. For any algebraic variety $\gV$, we define $\F_{\gV}:=\{f_\Theta(x)= \left\langle \Theta, \widetilde{x} ^{\otimes p} \right \rangle: \Theta \in \mathcal{V}  \}$ as the polynomial family parameterized by $\Theta \in \gV$. We say $\F_{\gV}$ is admissible\footnote{Admissible means the dimension of $\mathcal{X}$ decreases by one when there is an additional linear constraint $\langle \Theta, X\rangle =0$} w.r.t. $\mathcal{X}$, if for any $\Theta \in \gV$, $\text{dim}(\mathcal{X} \cap \{X \in \mathcal{X} : \langle X,  \Theta\rangle =0\rangle\}  )< \text{dim}(\mathcal{X})=d $. 
	%
	We define the dimension $D$ of the family to be the dimension of $\gV$.
\end{definition}

The following theorem shows that to learn an admissible polynomial family, the sample complexity only scales with the algebraic dimension of the polynomial family.

\begin{theorem}[\cite{huang2021optimal}]
	\label{thm:ag}
	Consider the polynomial family $\F_{\gV}$ of dimension $D$.  For $n \ge 2D$, there exists a Lebesgue-measure zero set $N \in \mathbb{R}^d \times \ldots \mathbb{R}^d$, such that if $(x_1,\cdots, x_n) \notin N$, for any $y_i$, there is a unique $f$ (or no such $f$) to the system of equations $y_i =f(x_i)$ for $f \in \F_{\gV}$ .
\end{theorem}

We give two important examples of admissible polynomial families with low dimension.
\begin{example} (Low-rank Polynomial of rank $k$)
	The function  $f\in \mathcal{F}_{\gV}$ is a polynomial with $k$ terms, that is 
	\begin{equation*}
	F(x) =  \sum_{i = 1}^k \lambda_i \langle v_i, x \rangle^{p_i}, \label{eq:poly:1}
	\end{equation*}
	where $p  = \max\{p_i\}$. The dimension of this family is upper bounded by $D \le dk$. Neural network with monomial/polynomial activation functions are low-rank polynomials.
	\label{ex:rank-k}
\end{example}
\begin{example}
	\label{example:qUx}
	The function $ f \in \mathcal{F}_\gV$ is of the form $f(x) = q(Ux)$, where $U \in \mathbb{R}^{k \times d}$ and $q$ is a degree $p$ polynomial. The polynomial $q$ and matrix $U$ are unknown. The dimension of this family is upper bounded by $ D \le  d (k+1)^p$. 
\end{example}

Next, we introduce the notion of positive measure.
\begin{definition}
	We say a measurable set $E \in \mathbb{R}^d$ is of positive measure if $\mu(E) > 0$, where $\mu$ is the standard Lebesgue measure on $\mathbb{R}^d$.
\end{definition}

If a measurable set $E$ satisfies $\mu(E)>0$, then there exists a procedure to draw samples from $E$, such that for any $N \subset \mathbb{R}^d$ of Lebesgue-measure zero, the probability that the sample falls in $N$ is zero. In fact, the sampling probability can be given by $\PP_{x \in \mathcal{N}(0,I_d)} (\cdot| x \in E)$. The intuition behind its definition is that for all admissible polynomial families, the set of $(x_1,\cdots,x_n)$ with "redundant information" about learning the parameter $\Theta$ is of Lebesgue-measure zero. Therefore, a positive measure set allows you to query randomly and avoids getting coherent measurements.

Next two theorems identify the sufficent conditions for efficiently learning deterministic MDPs under polynomial realizability. Specifically, under online RL setting, we require the strong assumption that the set $\{\phi_h(s, a) | a \in \gA\}$ is of positive measure for all $h \in [H]$ and all $s \in \gS$, while under generative model setting, we only require the union set $\bigcup_{s \in \gS} \{\phi_h(s, a) | a \in \gA\}$ to be of positive measure for all $h \in [H]$. The algorithms for solving the both cases are summarized in Algorithms~\ref{alg:querypoly} and~\ref{alg:onlinerlpoly}.

\begin{assumption}[Polynomial Realizability]
	\label{assump:pr}
	For all $ h \in [H]$,  $Q^\ast_h(s_h, a_h)$, viewed as the function of $\phi_h(s_h, a_h)$, lies in some admissible polynomial family $\F_{\gV_h}$ with dimension bounded by $D$.
\end{assumption}

\begin{theorem} For the generative model setting, assume that the set $\{\phi_h(s,a) \mid s\in \gS, a \in \gA \}$ is of positive measure at any level $h$. Under the polynomial realizability, Algorithm \ref{alg:querypoly} almost surely learns the optimal policy $\pi^\ast$ with at most  $N = 2 D H$ samples.
	\label{thm:gen-rl}
\end{theorem}

\begin{theorem}
	\label{thm:online-rl}
	For the online RL setting, assume that $\{\phi_h(s, a) \mid a\in \gA\}$ is of positive measure for every state $s$ at every level $h$. Under polynomial realizability, within $T = 2 DH$ episodes, 
	Algorithm \ref{alg:onlinerlpoly} learns the optimal policy $\pi^\ast$ almost surely.
	\label{thm4} 
\end{theorem}

\begin{algorithm*}[h]\caption{Dynamic programming under generative model settings}\label{alg:querypoly}
	\begin{algorithmic}[1]
		\For{$h=H, \cdots, 1$}
		\State Sample $2D$ points $\{ \phi_h(s_h^{(i)}, a_h^{(i)})\}_{i=1}^{2D}$ according to $\PP_{x \in \mathcal{N}(0,I_d)} (\cdot| x \in E_h)$ where $E_h = \{\phi_h(s,a) \mid s\in \gS, a \in \gA \}$.
		\State Query the generative model with state-action pair $(s_h^{(i)}, a_h^{(i)})$ at level $h$ for $i=1,\dots, 2D$, and observe the next state $\tilde{s}_h^{(i)}$ and reward $r_h^{(i)}$.
		\State Solve for $Q^*_h$ with the $2D$ equations $Q^*_h (s_h^{(i)}, a_h^{(i)}) = r_h^{(i)} + V^*_{h+1} (\tilde{s}_h^{(i)})$.
		\State Set $\pi^*_h(s) = \argmax_{a}  Q^*_h (s,a)$ and $V_h^*(s) = \max_{a}  Q^*_h (s,a)$.
		\EndFor
		\State \textbf{Output} $\pi^*$
	\end{algorithmic}
\end{algorithm*}

\begin{algorithm*}[h]\caption{Dynamic programming under online RL settings}\label{alg:onlinerlpoly}
	\begin{algorithmic}[1]
		\For{$h=H, \cdots, 1$}
		\State Fix any action sequence $a_1,\cdots,a_{h-1}$.
		\State Play $a_1,\cdots,a_{h-1}$ for the first $h-1$ levels and reach a state $s_{h}$. Sample $2D$ points $\{ \phi_h(s_h, a_h^{(i)})\}_{i=1}^{2D}$ according to $\PP_{x \in \mathcal{N}(0,I_d)} (\cdot| x \in E_h)$ where $E_h = \{\phi_h(s_h,a) \mid  a \in \gA \}$.
		
		\State Play $a_h^{(i)}$ at $s_h$ for $i=1,\dots, 2D$, and observe the next state $\tilde{s}_h^{(i)}$ and reward $r_h^{(i)}$.
		\State Solve for $Q^*_h$ with the $2D$ equations $Q^*_h (s_h^{(i)}, a_h^{(i)}) = r_h^{(i)} + V^*_{h+1} (\tilde{s}_h^{(i)})$.
		\State Set $\pi^*_h(s) = \argmax_{a}  Q^*_h (s,a)$ and $V_h^*(s) = \max_{a}  Q^*_h (s,a)$.
		\EndFor
		\State \textbf{Output} $\pi^*$
	\end{algorithmic}
\end{algorithm*}

We remark that our Theorem~\ref{thm4} for learning MDPs under the online RL setting relies on a very strong assumption that allows the learner to explore randomly for any state. However, this assumption is necessary in some sense, as is suggested by our lower bound example in the next subsection.

%% file: sections/sec_lb_det_mdp.tex

\def\cA{\mathcal{A}}
\def\cC{\mathcal{C}}
\def\cF{\mathcal{F}}
\def\cM{\mathcal{M}}

\newcommand{\indict}{\mathbb{I}}


\subsection{Necessity of Generic Feature Maps in Online RL}
\label{sec:det_lb}

In this section, we consider lower bounds for learning deterministic MDPs with polynomial realizable $Q^*$ under online RL setting.
Our goal is to show that in the online setting the generic assumption on the feature maps $\phi_h(s, \cdot)$ is necessary. 
On the contrary, under the generative model setting one can efficiently learn the MDPs without such a strong assumption, since at every level $h$ the we can set the state arbitrarily.

\paragraph{MDP construction} We briefly introduce the intuition of our construction. Consider a family of MDPs with only two states $\gS = \{S_{\text{good}}, S_{\text{bad}}\}$.
we set the feature map $\phi_h$ such that, for the good state $S_{\text{good}}$,  it allows the learner to explore randomly, i.e., $\{\phi_h(S_{\text{good}}, a) \mid a \in \gA \}$ is of postive measure.

However, for the bad state $S_{\text{bad}}$, all actions are mapped to some restricted set, which forbids random exploration, i.e., $\{\phi_h(S_{\text{bad}}, a) \mid a \in \gA\}$ is measure zero. This is illustrated in Figure~\ref{fig:generic}.

Specifically, at least $\Omega(d^p)$ actions are needed to identify the groud-truth polynomial of $Q_h^*$ for $S_{\text{bad}}$, while $O(d)$ actions suffice for $S_{\text{good}}$.

The transition $\PP_h$ is constructed as $ \PP_h(s_{\text{bad}}   | s, a) = 1 \text{ for all } s\in \gS, a \in \gA$, which means it is impossible for the online scenarios to reach the good state for $h>1$.

\begin{figure}[htbp]
\vspace{-0.15in}
\begin{center}
    \includegraphics[width=0.7\textwidth]{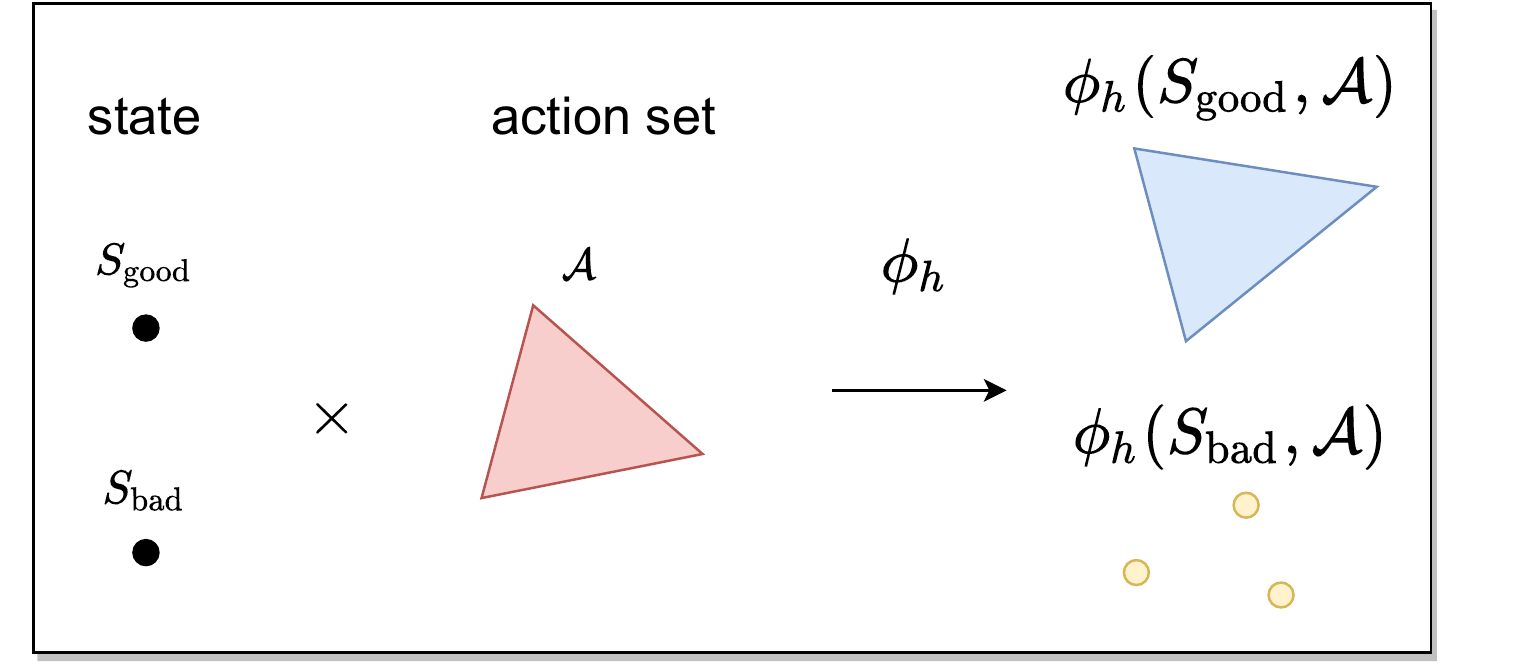}
    \caption{ An illustration of the hard case for deterministic MDPs with polynomial realizable $Q^*$. The image of the feature map $\phi_h$ at $S_{\text{good}}$ is of positive measure, while the image of $\phi_h$ at $S_{\text{bad}}$ is not. This makes it difficult to learn under the online RL setting.}
        \label{fig:generic}
\end{center}
\end{figure}

\begin{theorem}
There exists a family of MDPs satisfying Assumption~\ref{assump:pr}, such that the set $\left\{\phi_h(s,a) \mid s\in \gS, a \in \gA \right\}$ is of positive measure at any level $h$, but for all $h$ there is some $s_{\mathrm{bad}} \in \gS$ such that $\{\phi_h(s_{\mathrm{bad}},a) \mid a \in \gA \}$ is measure zero.
Under the online RL setting, any algorithm needs to play at least $\Omega(d^p)$ episodes to identify the optimal policy.
On the contrary, under the generative model setting, only $O(d) H$ samples are needed.
\end{theorem}

%% file: sections/conclusion.tex
In this paper, we consider neural network and polynomial function approximations in the simulator and online settings. To our knowledge, this is the first paper that shows sample-efficient reinforcement learning is possible with neural net function approximation. Our results substantially improve upon what can be achieved with existing results that primarily rely on embedding neural networks into linear function classes. The analysis reveals that for function approximations that allows for efficient $\ell_\infty$ recovery, such as two layer neural networks and admissible polynomial families, reinforcement learning can be reduced to parameter recovery problems, as well-studied in theories for deep learning, phase retrieval, and etc. Our method can also be potentially extended to handle three-layer and deeper neural networks, with advanced tools in \cite{fornasier19a,fornasier19b}.

Our results for polynomial activation require deterministic transitions, since we cannot handle how noise propagates in solving polynomial equations. We leave to future work an in-depth study of the stability of roots of polynomial systems with noise, which is a fundamental mathematical problem and even unsolved for homogeneous polynomials. In particular, noisy tensor decomposition approaches combined with zeroth-order optimization may allow for stochastic transitions~\cite{huang2021optimal}.

In the online RL setting, we can only show efficient learning under a very strong yet necessary assumption on the feature mapping. We leave to future work identifying more realistic and natural conditions which permit efficient learning in the online RL setting.

Finally, in future work, we hope to consider deep neural networks where parameter recovery or $\ell_\infty$ error is unattainable, and deep reinforcement learning with representation learning~\cite{yang2020provable,du2020few}.

%% file: acknowledgement.tex
\section*{Acknowledgment}
JDL acknowledges support of the ARO under MURI Award W911NF-11-1-0303,  the Sloan Research Fellowship, NSF CCF 2002272, and an ONR Young Investigator Award. QL
is supported by NSF 2030859 and the Computing Research
Association for the CIFellows Project. SK acknowledges funding from the NSF Award CCF-1703574 and the ONR award N00014-18-1-2247. BH is supported by the Elite Undergraduate Training Program of School of Mathematical Sciences at Peking University. 

%% file: appendix/nn_recovery.tex
\section{Omitted Proofs in Section~\ref{sec:nn_simulator}}

\subsection{Proof of Section~\ref{sec:det_nn}}\label{sec:proof_det_nn}

\begin{theorem}[Formal statement of Theorem~\ref{thm:det_nn}]\label{thm:det_nn_formal}
Consider MDP $\mdps$ where the transition is deterministic. Assume the function class in Definition~\ref{def:nn_class} satisfies Assumption~\ref{asp:realizability} and Assumption~\ref{asp:dense_feature}. For any $t \in (0,1)$, if $d \geq \Omega(\log (B_W / \lambda))$ and $n \geq d \cdot \poly(\kappa,k,\lambda, B_W, B_{\phi}, H,\log(d/t))$, then with probability at least $1-t$ Algorithm~\ref{alg:det_nn} returns the optimal policy $\pi^\ast$.
\end{theorem}

\begin{proof}
Use $\pi^\ast_1,\dots,\pi^\ast_H$ to denote the global optimal policy. We prove that Algorithm~\ref{alg:det_nn} learns $\pi_h^\ast$ from $h = H$ to $h = 1$.

At level $H$, the query obtains exact $Q_H^\ast(s,a)$. Therefore by Theorem~\ref{thm:exact_nn}, $\hat Q_H = Q_H^\ast$ and thus the optimal planning finds $\pi_H = \pi_H^\ast$. Suppose we have learned $\pi_{h+1}^\ast, \dots, \pi_H^\ast$ at level $h$. Due to deterministic transition, the query obtains exact $Q_h^\ast(s,a)$. Therefore by Theorem~\ref{thm:exact_nn}, $\hat Q_h = Q_h^\ast$ and thus the optimal planning finds $\pi_h = \pi_h^\ast$. Recursively applying this process to $h=1$, we complete the proof.
\end{proof}
\subsection{Proof of Section~\ref{sec:policy_nn}}
\label{sec:proof_policy_nn}

\begin{theorem}[Formal statement of Theorem~\ref{thm:policy_nn}]\label{thm:policy_nn_formal}
Assume the function class in Definition~\ref{def:nn_class} satisfies Assumption~\ref{asp:realizability}, Assumption~\ref{asp:dense_feature} and is policy complete. For any $\epsilon > 0$ and $t \in (0,1)$ such that $d \geq \Omega(\log (B_W B_\phi/\epsilon))$, if $n \geq \epsilon^{-2} \cdot d \cdot \poly(\kappa,k, B_W, B_{\phi}, H, \log(d/t))$, then with probability at least $1-t$ Algorithm~\ref{alg:simulator_nn_policy_complete} returns a policy $\pi$ such that $V^\ast - V^{\pi} \leq \epsilon$.
\end{theorem}

\begin{proof}
Use $\pi^\ast_1,\dots,\pi^\ast_H$ to denote the global optimal policy.
We prove for all $s \in \states$,
\begin{align*}
    V^{\pi^\ast_h,\pi^\ast_{h+1},\dots,\pi^\ast_H}_h(s) - V^{\pi_h,\pi_{h+1},\dots,\pi_H}_h(s) \leq  \frac{(H-h+1)\epsilon}{H}.
\end{align*}

At level $H$, let $e_H(s^i_H,a^i_H) = r_H(s^i_H,a^i_H)-Q^\ast_H(s^i_H,a^i_H)$, then $e_H(s^i_H,a^i_H) = 0$. 
From Theorem~\ref{thm:mom_recover_nn}, we have $\hat Q_H(s,a) := v_H^\top \sigma(W_H \phi(s,a))$ satisfies $|\hat Q_H(s,a) - Q^\ast_H(s,a)| \leq \frac{\epsilon}{2H}$ for all $s \in \states, a \in \actions$. Therefore for all $s \in \states$, 
\begin{align*}
    V_H^{\ast}(s) - V^{\pi_H}_H(s)
    = &~ \E_{a \sim \pi^\ast_H}[Q^\ast_H(s,a)] - \E_{a \sim \pi^\ast_H}[\hat Q_H(s,a)]\\
    &~ + \E_{a \sim \pi^\ast_H}[\hat Q_H(s,a)] - \E_{a \sim \pi_H}[ \hat Q_H(s,a)]\\
    &~ + \E_{a \sim \pi_H}[ \hat Q_H(s,a)] - \E_{a \sim \pi_H}[Q^\ast_H(s,a)]\\
    \leq &~ \frac{\epsilon}{H}
\end{align*}
where in the second step we used $\E_{a \sim \pi^\ast_H}[\hat Q_H(s,a)] \leq \E_{a \sim \pi_H}[ \hat Q_H(s,a)]$ by optimality of $\pi_H$ and $|\hat Q_H(s,a) - Q^\ast_H(s,a)| \leq \frac{\epsilon}{2H}$.

Suppose we have learned policies $\pi_{h+1},\dots,\pi_H$, we use $\tilde \pi_{h}$ to denote the optimal policy of $Q^{\pi_{h+1},\dots,\pi_H}_h(s,a)$. Let 
$$e_h(s^i_h,a^i_h) = {\hat Q}_h^i-Q_h^{\pi_{h+1},\dots,\pi_H}(s^i_h,a^i_h)$$
then $e_h(s^i_h,a^i_h)$ is zero mean $H^2$ sub-Gaussian (notice that ${\hat Q}_h^i$ is unbiased estimate of $Q_h^{\pi_{h+1},\dots,\pi_H}(s^i_h,a^i_h)$, and ${\hat Q}_h^i \leq O(H)$). From Theorem~\ref{thm:mom_recover_nn}, we have $\hat Q_h(s,a) = v_h^\top \sigma(W_h \phi(s,a))$ satisfies $|\hat Q_h(s,a) - Q_h^{\pi_{h+1},\dots,\pi_H}(s,a)| \leq \frac{\epsilon}{2H}$ for all $s \in \states, a \in \actions$. Therefore for all $s \in \states$, 
\begin{align*}
    &~ V^{\tilde \pi_h,\pi_{h+1},\dots,\pi_H}_h(s) - V^{\pi_h, \pi_{h+1},\dots,\pi_H}_h(s)\\
    = &~ \E_{a \sim \tilde \pi_h}[Q_h^{\pi_{h+1},\dots,\pi_H}(s,a)] - \E_{a \sim \tilde \pi_h}[\hat Q_h(s,a)]\\
    &~ + \E_{a \sim \tilde \pi_h}[\hat Q_h(s,a)] - \E_{a \sim \pi_h}[ \hat Q_h(s,a)]\\
    &~ + \E_{a \sim \pi_h}[ \hat Q_h(s,a)] - \E_{a \sim \pi_h}[ Q_h^{\pi_h, \pi_{h+1},\dots,\pi_H}(s,a)]\\
    \leq &~ \frac{\epsilon}{H}
\end{align*}
where in the second step we used $\E_{a \sim \tilde \pi_h}[\hat Q_h(s,a)] \leq \E_{a \sim \pi_h}[ \hat Q_h(s,a)]$ by optimality of $\pi_h$ and $|\hat Q_h(s,a) - Q_h^{\pi_{h+1},\dots,\pi_H}(s,a)| \leq \frac{\epsilon}{2H}$.

It thus follows that
\begin{align*}
    V^{\pi^\ast_h,\pi^\ast_{h+1},\dots,\pi^\ast_H}_h(s) - V^{\pi_h,\pi_{h+1},\dots,\pi_H}_h(s) = &~ V^{\pi^\ast_h,\pi^\ast_{h+1},\dots,\pi^\ast_H}_h(s) - V^{\pi^\ast_h,\pi_{h+1},\dots,\pi_H}_h(s)\\
    &~ + V^{\pi^\ast_h,\pi_{h+1},\dots,\pi_H}_h(s) - V^{\tilde \pi_h,\pi_{h+1},\dots,\pi_H}_h(s)\\
    &~ + V^{\tilde \pi_h,\pi_{h+1},\dots,\pi_H}_h(s) - V^{\pi_h,\pi_{h+1},\dots,\pi_H}_h(s)\\
    \leq &~ V^{\pi^\ast_h,\pi^\ast_{h+1},\dots,\pi^\ast_H}_h(s) - V^{\pi^\ast_h,\pi_{h+1},\dots,\pi_H}_h(s) +  \frac{\epsilon}{H}\\
    \leq &~ \cdots \\
    \leq &~ \frac{(H-h+1)\epsilon}{H}.
\end{align*}
where in the second step we use $V^{\pi^\ast_h,\pi_{h+1},\dots,\pi_H}_h(s) \leq V^{\tilde \pi_h,\pi_{h+1},\dots,\pi_H}_h(s)$ from optimality of $\tilde \pi_h$.
Repeating this argument to $h=1$ completes the proof
\end{proof}

\subsection{Proof of Section~\ref{sec:bellman_nn}}
\label{sec:proof_bellman_nn}

\begin{theorem}[Formal statement of Theorem~\ref{thm:bellman_nn}]\label{thm:bellman_nn_formal}
Assume the function class in Definition~\ref{def:nn_class} satisfies Assumption~\ref{asp:realizability}, Assumption~\ref{asp:dense_feature}, and is Bellman complete. For any $\epsilon > 0$ and $t \in (0,1)$ such that $d \geq \Omega(\log (B_W B_\phi/\epsilon))$, if $n \geq \epsilon^{-2} \cdot d \cdot \poly(\kappa,k, B_W, B_{\phi}, H, \log(d/t))$, then with probability at least $1-t$ Algorithm~\ref{alg:simulator_nn_bellman_complete} returns a policy $\pi$ such that $V^\ast - V^{\pi} \leq \epsilon$.
\end{theorem}

\begin{proof}
Use $\pi^\ast_1,\dots,\pi^\ast_H$ to denote the global optimal policy.
We prove 
\begin{align}\label{eq:Q*-Qhat}
    |\hat Q_h(s,a) - Q^\ast_h(s,a)| \leq \frac{(H-h+1)\epsilon}{H}
\end{align}
for all $s \in \states, a \in \actions$.

At level $H$, let $$e_H(s^i_H,a^i_H) = r_H(s^i_H,a^i_H)-Q^\ast_H(s^i_H,a^i_H)$$
then $e_H(s^i_H,a^i_H) = 0$
.
From Theorem~\ref{thm:mom_recover_nn}, we have $\hat Q_H(s,a) := v_H^\top \sigma(W_H \phi(s,a))$ satisfies $|\hat Q_H(s,a) - Q^\ast_H(s,a)| \leq \frac{\epsilon}{H}$ for all $s \in \states, a \in \actions$. 

Suppose we have learned $\hat Q_{h+1}(s,a)$ with $|\hat Q_{h+1}(s,a) - Q^\ast_{h+1}(s,a)| \leq \frac{(H-h)\epsilon}{H}$. At level $h$, let $$e_h(s^i_h,a^i_h) = r_h(s^i_h,a^i_h) + \hat V_{h+1}(s_{h+1}^i) -\T_{h}(\hat Q_{h+1})(s^i_h,a^i_h)$$
then $e_h(s^i_h,a^i_h)$ is zero mean $H^2$ sub-Gaussian (notice that $r_h(s^i_h,a^i_h) + \hat V_{h+1}(s_{h+1}^i)$ is unbiased estimate of $\T_{h}(\hat Q_{h+1})(s^i_h,a^i_h)$, and $r_h(s^i_h,a^i_h) + \hat V_{h+1}(s_{h+1}^i) \leq O(H)$). From Theorem~\ref{thm:mom_recover_nn}, we have $\hat Q_h(s,a) := v_h^\top \sigma(W_h \phi(s,a))$ satisfies $|\hat Q_h(s,a) - \T_{h}(\hat Q_{h+1})(s^i_h,a^i_h)| \leq \frac{\epsilon}{H}$ for all $s \in \states, a \in \actions$. Therefore
\begin{align*}
    |\hat Q_h(s,a) - Q^\ast_h(s,a)| \leq &~ |\hat Q_h(s,a) - \T_{h}(\hat Q_{h+1})(s,a)| + |\T_{h}(\hat Q_{h+1})(s,a) - Q^\ast_h(s,a)|\\
    \leq &~ \frac{\epsilon}{H} + \max_{s \in \states, a \in \actions} |\hat Q_{h+1}(s,a) - Q^\ast_{h+1}(s,a)|\\
    \leq &~ \frac{(H-h+1)\epsilon}{H}
\end{align*}
holds for all $s \in \states, a \in \actions$.

It thus follows that
for all $s_1 \in \states$, 
\begin{align*}
    V^{\pi^\ast_1,\dots,\pi^\ast_H}_h(s_1) - V^{\pi_1,\dots,\pi_H}_h(s_1) = &~ \E_{a \sim \pi^\ast_1}[Q_1^\ast(s_1,a)] - \E_{a \sim \pi_1}[Q_1^{\pi_2,\dots,\pi_H}(s_1,a)]\\
    \leq &~ \E_{a \sim \pi^\ast_1}[\hat Q_1 (s_1,a)] - \E_{a \sim \pi_1}[Q_1^{\pi_2,\dots,\pi_H}(s_1,a)] + \epsilon\\
    \leq &~ \E_{a \sim \pi_1}[\hat Q_1 (s_1,a)- Q_1^{\pi_2,\dots,\pi_H}(s_1,a)] + \epsilon \\
    \leq &~ \E_{a \sim \pi_1}[Q_1^\ast (s_1,a) - Q_1^{\pi_2,\dots,\pi_H}(s_1,a)] + 2 \epsilon\\
    \leq &~ \E_{a \sim \pi_1}\E_{s_2 \sim \PP(\cdot|s,a)}[ V_2^{\pi^\ast_2,\dots,\pi^\ast_H}(s_2) - V_2^{\pi_2,\dots,\pi_H}(s_2)] + 2\epsilon\\
    \leq &~ \cdots\\
    \leq &~ 2H\epsilon
\end{align*}
where the first step comes from definition of value function; the second step comes from Eq.~\eqref{eq:Q*-Qhat}; the third step comes from optimality of $\pi_1$; the fourth step comes from Eq.~\eqref{eq:Q*-Qhat}; the fifth step comes from Bellman equation. The proof is complete by rescaling $\epsilon \leftarrow \epsilon/H$.
\end{proof}

\subsection{Proof of Section~\ref{sec:gap_Q_star}}\label{sec:proof_gap_nn}

With gap condition, either Algorithm~\ref{alg:simulator_nn_policy_complete} or Algorithm~\ref{alg:simulator_nn_bellman_complete} will work as long as we select $\epsilon \approx \rho$. The following displays an adaption from Algorithm~\ref{alg:simulator_nn_policy_complete}.

\begin{algorithm*}[h]\caption{Learning realizable $Q^\ast$ with optimality gap}\label{alg:gap_nn}
	\begin{algorithmic}[1]
		\For{$h=H,\ldots 1$}
		\State Sample $x^i_h, i \in [n]$ from standard Gaussian $\mathcal{N}(0,I_d)$ 
		\For{$i \in [n]$}
		\If{$\|x^i_h\| \leq \delta_\phi$}
		\State Find $(s^i_h,a^i_h) \in \phi^{-1}(x^i_h)$ and locate the state $s^i_h$ in the generative model
		\State Pull action $a^i_h$ and use $\pi_{h+1}, \dots,\pi_{H}$ as the roll-out to collect rewards $r_{h}^{(i)} , \dots, r_{H}^{(i)}$
		\State Construct unbiased estimation of ${ Q}_h^{\pi_{h+1}, \dots,\pi_{H}}(s^i_h,a^i_h)$
		$${\hat Q}_h^i \leftarrow r_{h}^{(i)} + \cdots + r_{H}^{(i)}$$
		\Else
		\State Let ${\hat Q}_h^i \leftarrow 0$.
		\EndIf
		\EndFor
		\State Compute $(v_h,W_h) \leftarrow \textsc{NeuralNetNoisyRecovery}(\{(x^i_h,{\hat Q}_h^i): i \in [n] \})$
		\State Set $\hat Q_h(s,a) \leftarrow  v_h^\top \sigma(W_h \phi(s,a))$
		\State Let $\pi_h(s) \leftarrow \argmax_{a \in \states}~ \hat Q_h(s,a)$
		\EndFor
		\State {\bf Return} $\pi_1,\dots,\pi_H$
	\end{algorithmic}
\end{algorithm*}

\begin{theorem}[Formal statement of Theorem~\ref{thm:gap_Q_star}]
Assume the function class in Definition~\ref{def:nn_class} satisfies Assumption~\ref{asp:realizability} and Assumption~\ref{asp:dense_feature}. Suppose $\rho > 0$ and $d \geq \Omega(\log (B_W B_\phi /\rho))$, for any $t \in (0,1)$, if $n = \frac{d}{\rho^2} \cdot \poly(\kappa,k, B_W, B_{\phi}, H, \log(d/t))$, then with probability at least $1-t$ Algorithm~\ref{alg:gap_nn} returns the optimal policy $\pi^\ast$.
\end{theorem}

\begin{proof}
Use $\pi^\ast_1,\dots,\pi^\ast_H$ to denote the global optimal policy. Similar to Theorem~\ref{thm:det_nn_formal}, we prove that Algorithm~\ref{alg:gap_nn} learns $\pi_h^\ast$ from $h = H$ to $h = 1$.

At level $H$, the algorithm uses $n = \frac{d}{\rho^2} \cdot \poly(\kappa,k,\log d, B_W, B_{\phi}, H, \log(1/t))$ trajectories to obtain $\hat{Q}_H$ such that $|\hat{Q}_H(s,a) - Q^\ast(s,a)| \leq \rho/4$ by Theorem~\ref{thm:mom_recover_nn}. Therefore 
\begin{align*}
    &~ V^\ast_H(s) - Q^\ast_H(s,\pi_H(s))\\
    \leq &~ Q^\ast_H(s,\pi^\ast_H(s)) - Q^\ast_H(s,\pi_H(s))\\
    \leq &~ Q^\ast_H(s,\pi^\ast_H(s)) - \hat{Q}_H(s,\pi^\ast_H(s)) + \hat{Q}_H(s,\pi^\ast_H(s)) - \hat{Q}_H(s,\pi_H(s))\\
    &~ + \hat{Q}_H(s,\pi_H(s)) - Q^\ast_H(s,\pi_H(s))\\
    \leq &~ \rho/2
\end{align*}
where the third inequality uses the optimality of $\pi_H(s)$ under $\hat{Q}_H$. Thus Definition~\ref{def:opt_gap} gives $\pi_H(s) = \pi^\ast_H(s)$.
Suppose we have learned $\pi_{h+1}^\ast, \dots, \pi_H^\ast$ at level $h$. We apply the same argument to derive $\pi_h = \pi^\ast_h$. Recursively applying this process to $h=1$, we complete the proof.
\end{proof}

\subsection{Neural network recovery}\label{sec:nn_recovery_guarantee}

This section considers recovering neural network $\langle v, \sigma(W x) \rangle$ from the following two models, where $B = \Omega(d \cdot \poly \log (d))$.
\begin{itemize}
    \item Noisy samples from
    \begin{align}\label{eq:noisy_model}
    x \sim \mathcal{N}(0,I_d), ~~ y = (\langle v, \sigma(W x) \rangle + \xi)\cdot \mathbbm{1}(\|x\| \leq B)
\end{align}
where $\xi$ is $\vartheta$ sub-Gaussian noise. 
    \item Noiseless samples from
    \begin{align}\label{eq:noiseless_model}
    x \sim \mathcal{N}(0,I_d), ~~ y = (\langle v, \sigma(W x) \rangle) \cdot \mathbbm{1}(\|x\| \leq B)
\end{align}
\end{itemize}

Recovering neural network has received comprehensive study in deep learning theory \cite{jsa15,zsj+17,glm17}. The analysis in this section is mainly based on the method of moments in \cite{zsj+17}. However, notice that the above learning tasks are different from those considered in \cite{zsj+17}, due to the presence of noise and the truncated signals. Therefore, additional considerations must be made in the analysis.

We consider more general homogeneous activation functions, specified by the assumptions that follow. 
Since the activation function is homogeneous, we assume $v_i \in \{\pm 1\}$ in the following without loss of generality.

\begin{assumption}[Property 3.1 of \cite{zsj+17}]\label{asp:1st_derivative_poly_bound}
Assume $\sigma'(x)$ is nonnegative and homogeneously bounded, i.e. $0 \leq \sigma'(x) \leq L_1 |x|^p$ for some constants $L_1 >0$ and $p \geq 0$.
\end{assumption}

\begin{definition}[Part of property 3.2 of \cite{zsj+17}]
Define	$\rho(z):=\min\{ \beta_0(z)-\alpha_0^2(z)-\alpha_1^2(z),\beta_2(z)-\alpha_1^2(z)-\alpha_2^2(z),\alpha_0(z)\alpha_2(z)-\alpha_1^2(z) \}$,  where  $\alpha_q(z):=\E_{x\sim \mathcal{N}(0,1)}[\sigma'(z x)x^q], q\in \{0,1,2\}$, and $\beta_q(z):=\E_{x\sim \mathcal{N}(0,1)}[(\sigma')^2( z x)x^q]$ for $q \in \{0,2\}$.
\end{definition}

\begin{assumption}[Part of property 3.2 of \cite{zsj+17}]\label{asp:rho>0}
The first derivative $\sigma'(z)$ satisfies that, for all $z>0$, we have $\rho(z)>0$.
\end{assumption}

\begin{assumption}[Property 3.3 of \cite{zsj+17}]
\label{asp:hessian_activation}
The second derivative $\sigma''(x)$ is either {\bf (a)} globally bounded or {\bf (b)} $\sigma''(x)=0$ except for finite points.
\end{assumption}

Notice that ReLU, squared ReLU, leaky ReLU, and polynomial activation function functions all satisfies the above assumption. 
We make the following assumption on the dimension of feature vectors, which corresponds to how features can extract information about neural networks from noisy samples. The dimension only has to be greater than a logarithmic term in $1/\epsilon$ and the norm of parameters.

\begin{assumption}[Rich feature]\label{asp:rich_feature}
Assume $d \geq \Omega(\log(B_W  /\epsilon))$.
\end{assumption}

First we introduce a notation from \cite{zsj+17}.
\begin{definition}\label{def:outer_product}
Define outer product $\Tilde{\otimes}$ as follows. For a vector $v \in \R^d$ and an identity matrix $I \in \R^{d \times d}$, 
$$v \Tilde{\otimes} I = \sum_{j=1}^d [ v \otimes e_j \otimes e_j + e_j \otimes v \otimes e_j + e_j \otimes e_j \otimes v].$$
For a symmetric rank-$r$ matrix $M = \sum_{i=1}^r s_i v_i v_i^\top$ and an identity matrix $I \in \R^{d \times d}$, $$M \Tilde{\otimes} I = \sum_{i = 1}^r s_i \sum_{j=1}^d \sum_{l=1}^6 A_{l,i,j}$$
where $A_{1,i,j} = v_i \otimes v_i \otimes e_j \otimes e_j$, $A_{2,i,j} = v_i  \otimes e_j \otimes v_i \otimes e_j$, $A_{3,i,j} = e_j \otimes v_i \otimes v_i \otimes  e_j$, $A_{4,i,j} = v_i \otimes e_j \otimes e_j \otimes v_i $, $A_{5,i,j} = e_j \otimes v_i \otimes e_j \otimes v_i$, $A_{6,i,j} = e_j \otimes e_j \otimes v_i \otimes v_i $.
\end{definition}
Now we define some moments.
\begin{definition}\label{def:moments}
Define $M_1,M_2,M_3,M_4,m_{1,i},m_{2,i},m_{3,i},m_{4,i}$ as follows:
\begin{align*}
    M_1 := &~ \E[y \cdot x]\\
    M_2 := &~ \E[y \cdot ( x \otimes x - I)]\\
    M_3 := &~ \E[y \cdot (x^{\otimes 3} - x \Tilde{\otimes} I)]\\
    M_4 := &~ \E[y \cdot (x^{\otimes 4} - (x \otimes x) \Tilde{\otimes} I + I\Tilde{\otimes} I )]\\
    \gamma_j(x) := &~ \E_{z \sim \mathcal{N}(0,1)}[\sigma(x \cdot z)z^j], \forall j \in 0,1,2,3,4\\
    m_{1,i} := &~ \gamma_1(\|w_i\|)\\
    m_{2,i} := &~ \gamma_2(\|w_i\|) - \gamma_0(\|w_i\|)\\
    m_{3,i} := &~ \gamma_3(\|w_i\|) - 3 \gamma_1(\|w_i\|)\\
    m_{4,i} := &~ \gamma_4(\|w_i\|) + 3 \gamma_0(\|w_i\|) - 6 \gamma_2 (\|w_i\|)
\end{align*}
The above expectations are all with respect to $x \sim \mathcal{N}(0,I_d)$ and $y = \langle v, \sigma(W x) \rangle$.
\end{definition}


\begin{assumption}[Assumption 5.3 of \cite{zsj+17}]\label{asp:nonvanish_moments}
Assume the activation function satisfies the followings:
\begin{itemize}
    \item If $M_i \neq 0$, then $m_{j,i} \neq 0$ for all $i \in [k]$.
    \item At least one of $M_3$ and $M_4$ is not zero.
    \item If $M_1 = M_3 = 0$, then $\sigma(z)$ is an even function.
    \item If $M_2 = M_4 = 0$, then $\sigma(z)$ is an odd function.
\end{itemize}
\end{assumption}

Now we state the theoretical result that recovers neural networks from noisy data.
\begin{theorem}[Neural network recovery from noisy data]\label{thm:mom_recover_nn}
Let the activation function $\sigma$ satisfies Assumption~\ref{asp:1st_derivative_poly_bound} and Assumption~\ref{asp:nonvanish_moments}. Let $\kappa$ be the condition number of $W$. Given $n$ samples from Eq.~\eqref{eq:noisy_model}. For any $t  \in (0,1)$ and $\epsilon \in (0,1)$ such that Assumption~\ref{asp:rich_feature} holds, if
$$n \geq \epsilon^{-2} \cdot d \cdot \poly(\kappa,k,\vartheta,\log (d/t))$$
then there exists an algorithm that takes $\Tilde{O}(n k d)$ time and returns a matrix $\hat W \in \R^{k \times d}$ and a vector $\hat v \in \{\pm 1\}^k$ such that with probability at least $1-t$, 
\begin{align*}
    \|\hat W - W \|_F \leq \epsilon \cdot \poly(k,\kappa) \cdot \|W\|_F, \text{ and } \hat v = v.
\end{align*}
\end{theorem}

The algorithm and proof are shown in Appendix~\ref{sec:mom}. By Assumption~\ref{asp:1st_derivative_poly_bound}, the following corollary is therefore straightforward.

\begin{corollary}
In the same setting as Theorem~\ref{thm:mom_recover_nn}. For any $t \in (0,1)$ and suppose $\|W\|_F \leq B_W$ and Assumption~\ref{asp:rich_feature} holds. Given $n$ samples from Eq.~\eqref{eq:noisy_model}. If
$$n \geq \epsilon^{-2} \cdot d \cdot \poly(\kappa,k,\log (d/t), B_W, B_{\phi}, \vartheta)$$
then there exists an algorithm that takes $\Tilde{O}(n k d)$ time and outputs a matrix $\hat W \in \R^{k \times d}$ and a vector $\hat v \in \{\pm 1\}^k$ such that with probability at least $1-t$, for all $\|x\|_2 \leq B_\phi$ 
\begin{align*}
    |\langle  \hat v, \sigma(\hat W x) \rangle - \langle  v, \sigma( W x) \rangle | \leq \epsilon.
\end{align*}
In particular, when $B_\phi = O(d\cdot \poly \log d) $ the following sample complexity suffices
\begin{align*}
    n \geq \epsilon^{-2} \cdot d^{O(1+p)} \cdot \poly(\kappa,k,\log (d/t), B_W,  \vartheta).
\end{align*}
\end{corollary}

Now we state the theoretical result that precisely recovers neural networks from noiseless data. The proof and method are shown in Appendix~\ref{sec:gd+mom}.

\begin{theorem}[Exact neural network recovery from noiseless data]\label{thm:exact_nn}
Let the activation function satisfies Assumption~\ref{asp:1st_derivative_poly_bound} and Assumption~\ref{asp:nonvanish_moments},  Assumption~\ref{asp:rho>0} and Assumption~\ref{asp:hessian_activation}(b).  Given $n$ samples from Eq.~\eqref{eq:noiseless_model}. For any $ t \in (0,1)$, suppose $d \geq \Omega(\log (B_W / \lambda))$ and
\begin{align*}
    n \geq d \cdot \poly(\kappa, k,\lambda,\log (d/t)),
\end{align*}
then there exists an algorithm that output exact $W$ and $v$ with probability at least $1-t$.
\end{theorem}

%% file: appendix/mom.tex
\subsubsection{Recover neural networks from noisy data}\label{sec:mom}

In this section we prove Theorem~\ref{thm:mom_recover_nn}. Denote $W = [w_1,\cdots,w_k]^\top$ where $w_i \in \R^d$ and $\Bar{w_i} = w_i/\|w_i\|_2$.

\begin{definition}\label{def:P_2_P_3}
Given a vector $\alpha \in \R^d$. Define $P_2 := M_{j_2}(I,I,\alpha,\cdots,\alpha)$ where $j_2 = \min \{j \geq 2: M_j \neq 0\}$ and $P_3 := M_{j_3}(I,I,I,\alpha,\cdots,\alpha)$ where $j_3 = \min \{j \geq 3: M_j \neq 0\}$.
\end{definition}

The method of moments is presented in Algorithm~\ref{alg:mom}. Here we sketch its ideas, and refer readers to \cite{zsj+17} for thorough explanations. There are three main steps. In the first step, it computes the span of the rows of $W$. By power method, Line~\ref{lin:power_method_V} finds the top-$k$ eigenvalues of $CI + \hat P_2$ and $CI- \hat P_2$. It then picks the largest $k$ eigenvalues from $CI + \hat P_2$ and $CI- \hat P_2$, by invoking $\textsc{TopK}$ in Line~\ref{lin:topk}. Finally it orthogonalizes the corresponding eigenvectors in Line~\ref{lin:orthogonalize_V} and finds an orthogonal matrix $V$ in the subspace spanned by $\{\Bar{w}_1,\dots,\Bar{w}_k\}$.

In the second step, the algorithm forms third order tensor $R_3 = P_s(V,V,V) \in \R^{k \times k \times k}$ and use the robust tensor decomposition method in \cite{kcl15} to find $\hat u$ that approximates $s_i V^\top \Bar{w}_i$ with unknown signs $s_i$.
In the third step, the algorithm determines $s$, $v$ and $w_i, i \in [k]$. Since the activation function is homogeneous, we assume $v_i \in \{\pm 1\}$ and $m_{j,i} = c_j \|w_i\|^{p+1}$ for universal constants $c_j$ without loss of generality. For illustration, we define $Q_1$ and $Q_2$ as follows.
\begin{align}
& Q_1 = M_{l_1}(I,\underbrace{ \alpha,\cdots, \alpha}_{(l_1 - 1)~ \alpha\text{'s}}) = \sum_{i=1}^k v_i c_{l_1} \| w_i\|^{p+1} ( \alpha^\top \ov{ w}_i)^{{l_1-1}} \ov{ w}_i,\label{eq:def_Q1} \\
& Q_2 = M_{l_2}(V,V,\underbrace{ \alpha,\cdots, \alpha}_{(l_2 - 2)~ \alpha\text{'s}}) = \sum_{i=1}^k v_i c_{l_2} \| w_i\|^{p+1} ( \alpha^\top \ov{ w}_i)^{{l_2-2}} (V^\top\ov{ w}_i)(V^\top\ov{ w}_i)^\top,\label{eq:def_Q2}
\end{align}
where $l_1 \geq 1$ such that $M_{l_1} \neq 0$ and $l_2 \geq 2$ such that $M_{l_2} \neq 0$ are specified later.
Then the solutions of the following linear systems
\begin{align}\label{eq:zr_star}
z^*  = \argmin_{z \in \mathbb{R}^k } \left\| \sum_{i=1}^k z_i s_i\ov{ w}_i  - Q_1 \right\|, ~~
 r^*  = \argmin_{ r \in \mathbb{R}^k } \left\|\sum_{i=1}^k r_i V^\top \ov  w_i (V^\top \ov  w_i)^\top - Q_2 \right\|_F.
\end{align}
are the followings
\begin{align*}
 z^*_i & = v_is_i^{l_1}   c_{l_1} \| w_i\|^{p+1} ( \alpha^\top s_i\ov{ w}_i)^{{l_1-1}}, ~~
r_i = v_i s_i^{l_2} c_{l_2} \| w_i \|^{p+1} ( \alpha^\top s_i\ov{ w}_i)^{{l_2-2}} .
\end{align*}
When $c_{l_1}$ and $c_{l_2} $ do not have the same sign, we can recover $v_i$ and $s_i$ by $v_i = \sign({r}^{\ast}_i c_{l_2} ), s_i = \sign(v_i {z}^\ast_i c_{l_1})$, and recover $w_i$ by
\begin{align*}
    {w}_i = \left(\left|\frac{z^\ast_i}{c_{l_1}(\alpha^\top s_i\ov{ w}_i )^{l_1-1})}\right|\right)^{1/(p+1)} \ov{ w}_i.
\end{align*}

In Algorithm~\ref{alg:mom}, we use $V \wh{u}_i$ to approximate $s_i \ov  w_i$, and use moment estimators $\wh{Q}_1$ and $\wh{Q}_2$ to approximate $Q_1$ and $Q_1$. Then the solutions $\wh{z},  \wh{r} $ to the optimization problems in Line~\ref{lin:linear_system} should approximate $z^*$ and $r^*$, due to robustness for solving linear systems. As such, the outputs $\wt{v},\wt{W}$ approximately recover the true model parameters.

\begin{algorithm}[H]\caption{Using method of moments to recover neural network parameters}\label{alg:mom}
\begin{algorithmic}[1]
\Procedure{$\textsc{NeuralNetNoisyRecovery}$}{$S = \{(x_i,y_i): i \in [n]\}$}
\State Choose $\alpha$ to be a random unit vector
\State Partition $S$ into $S_1,S_2,S_3,S_4$ of equal size
\State $\hat P_2 \leftarrow \E_{S_1}[P_2]$, $C \leftarrow 3\|P_2\|$, $T \leftarrow \log(1/\epsilon)$ \label{lin:P_2}
\State Choose $\hat V_1^{(0)}, \hat V_1^{(0)} \in \R^{d \times  k}$ to be random matrices
\Comment{Estimate subspace $V$}
\For{$t = 1,\dots,T$}
    \State $\hat V_1^{(t)} \leftarrow \mathrm{QR}(C \hat V_1^{(t-1)} + \hat P_2 \hat V_1^{(t-1)}), \hat V_2^{(t)} \leftarrow \mathrm{QR}(C \hat V_2^{(t-1)} - \hat P_2 \hat V_2^{(t-1)})$ \label{lin:power_method_V}
\EndFor
\For{j = 1,2}
    \State $\hat V_1^{(T)} \leftarrow [\hat V_{j,1},\cdots, \hat V_{j,k}]$
    \For{$i \in [k]$}
        \State $\lambda_{j,i} \leftarrow |\hat V_{j,i} \hat P_2 \hat V_{j,i}|$ \label{lin:abs_eigenvalues_V}
    \EndFor
\EndFor
\State $\pi_1,\pi_2,k_1,k_2 \leftarrow \textsc{TopK}(\lambda,k)$\label{lin:topk}
\For{j = 1,2}
    \State $V_j \leftarrow [\hat V_{j,\pi_j(1)},\cdots, \hat V_{j,\pi_j(k_j)}]$
\EndFor
\State $\Tilde{V}_2 \leftarrow \textsc{QR}((I-V_1V_1^\top) V_2)$, $V \leftarrow [V_1,\Tilde{V}_2]$ \label{lin:orthogonalize_V}
\State $\hat R_3 \leftarrow \E_{S_2}[P_3(V,V,V)]$, $\{\hat u_i\}_{i \in [k]} \leftarrow \textsc{TensorDecomposition}(\hat R_3)$ \label{lin:R_3} \Comment{Learn $s_iV^\top \Bar{w}_i$}
\If{$M_1 = M_3 = 0$}
    \State $l_1,l_2 = \min \{j \in \{2,4\}: M_j \neq 0\}$
\ElsIf{$M_2 = M_4 = 0$}
    \State $l_1 \leftarrow \min \{j \in \{1,3\}: M_j \neq 0\}$, $l_2 \leftarrow 3$
\Else
    \State $l_1 \leftarrow \min \{j \in \{1,3\}: M_j \neq 0\}$, $l_2 = \min \{j \in \{2,4\}: M_j \neq 0\}$
\EndIf
\State $\hat Q_1 \leftarrow \E_{S_3}[Q_1]$, $\hat Q_2 \leftarrow \E_{S_4}[Q_2]$ \label{lin:Q}
\State $\hat z \leftarrow \argmin_z \|\sum_{i=1}^k z_i V \hat u_i - \hat Q_1\|$, $\hat r \leftarrow \argmin_r \|\sum_{i=1}^k r_i \hat u_i \hat u_i - \hat Q_2\|_F$\label{lin:linear_system}
\For{$i = 1,\dots,k$} \Comment{Learn parameters $v,W$}
    \State ${\hat v}_i \leftarrow \mathrm{sign}(\hat r_i c_{l_2})$, ${\hat s}_i \leftarrow \mathrm{sign}({\hat v}_i \hat z_i c_{l_1})$
    \State ${\hat w}_i  \leftarrow {\hat s}_i(|\frac{\hat z_i}{c_{l_1}(\alpha^\top V \hat u_i)^{l_1-1})}|)^{1/(p+1)}V\hat u_i$ \label{lin:wi0}
\EndFor
\State ${\hat W} \leftarrow [{\hat w}_i ,\cdots,{\hat w}_k ]$
\State {\bf Return} $({\hat v},{\hat W})$
\EndProcedure

\end{algorithmic}
\end{algorithm}

Since Algorithm~\ref{alg:mom} carries out the same computation as \cite{zsj+17}, the computational complexity is the same. The difference of sample complexity comes from the noise $\xi$ in the model and the truncation of standard Gaussian. The proof entails bounding the error in estimating $P_2$ in Line~\ref{lin:P_2}, $R_3$ in Line~\ref{lin:R_3} and $Q_1,Q_2$ in Line~\ref{lin:Q}. In the following, unless further specified, the expectations are all with respect to $x \sim \mathcal{N}(0,I_d)$ and $y \sim (\langle v, \sigma(W x) \rangle + \xi) \cdot \mathbbm{1}(\|x\|\leq B)$.

\begin{lemma}\label{lem:error_P2}
Let $\hat P_2$ be computed in Line~\ref{lin:P_2} of Algorithm~\ref{alg:mom} and $P_2$ defined in Definition~\ref{def:P_2_P_3}. Suppose $m_0=\min_{i\in[k]} \{|m_{j_2,i} |^2 (\ov w_i^{\top}  \alpha)^{2(j_2-2)} \}$ and
\begin{align*}
|S| \gtrsim d \cdot  \poly(\kappa, \vartheta, \log (d/t))  / (\epsilon^2 m_0)
\end{align*}
then with probability at least $1-t$,
\begin{align*}
    \|P_2 - \hat P_2\| \lesssim \epsilon \sum_{i = 1}^k|v_i m_{j_2,i}(\Bar{w_i}^\top \alpha)^{j_2-2}| + \epsilon.
\end{align*}
\end{lemma}

\begin{proof}
It suffices to bound $\|M_2 - \hat M_2\|$, $\|M_3(I,I,\alpha) - \hat M_3(I,I,\alpha)\|$ and $\|M_4(I,I,\alpha,\alpha) - \hat M_4(I,I,\alpha,\alpha)\|$. The main strategy is to bound all relevant moment terms and to invoke Claim~\ref{cla:modified_bernstein_non_zero}.

Specifically, we show that with probability at least $1-t/4$,
\begin{align}\label{eq:M2(I,I)}
    \|M_2 - \wh{M}_2\| \lesssim \epsilon \sum_{i=1}^k  |v_i m_{2,i} | + \epsilon.
\end{align}

\begin{align}\label{eq:M3(I,I,a)}
    \|M_3(I,I, \alpha) - \wh{M}_3(I,I, \alpha)\| \lesssim \epsilon \sum_{i=1}^k  |v_i m_{3,i} (\ov w_i^{\top}  \alpha)| + \epsilon.
\end{align}

\begin{align}\label{eq:M4(I,I,a,a)}
    \|M_4(I,I, \alpha, \alpha) - \wh{M}_4(I,I, \alpha, \alpha)\| \lesssim \epsilon \sum_{i=1}^k  |v_i m_{4,i}| (\ov w_i^{\top}  \alpha)^2 + \epsilon.
\end{align}

Recall that for sample $(x_j,y_j) \in S$, $y_j = \sum_{i=1}^k v_i \sigma(w_i^\top x_j) + \xi_j$ where $\xi_j$ is independent of $x_j$. 
Consider each component $i \in [k]$. 
Define $C_i(x_j), B_i(x_j) \in \R^{d \times d}$ as follows:
\begin{align*}
    B_i(x_j)
    = &~ (\sigma(w_i^\top x_j) + \xi_j) \cdot (x_j^{\otimes 4} - (x_j \otimes x_j) \Tilde{\otimes} I + I \Tilde{\otimes} I)(I,I, \alpha, \alpha) \\
    = &~ (\sigma(w_i^\top x_j) + \xi_j) \cdot (( x^\top  \alpha)^2 x^{\otimes 2} - ( \alpha^\top  x)^2 I - 2( \alpha^\top x)( x \alpha^\top+ \alpha  x^\top) -  x x^\top + 2 \alpha  \alpha^\top + I),
\end{align*}
and $C_i(x_j) = \mathbbm{1}({\|x_j\| \leq B}) \cdot B_i(x_j)$.
Then from Claim~\ref{cla:moments} we have $\E[B_i(x_j)] =  m_{4,i} (\ov w_i^{\top}  \alpha)^2 \ov w_i \ov w_i^{\top}$.
We calculate
\begin{align*}
&~ \sigma(w_i^\top x_j)  \cdot (x_j^{\otimes 4} - (x_j \otimes x_j) \Tilde{\otimes} I + I \Tilde{\otimes} I)(I,I, \alpha, \alpha)\\
\lesssim & (| w_i^{\top } x_j |^{p+1} + |\phi(0)|) \cdot (( x_j^\top  \alpha)^2\| x_j \|^2 + 1 +\| x_j\|^2 + ( \alpha^\top  x_j )^2)\\
\lesssim &~ |w_i|^{p+1}\cdot |x_j |^{p+5},
\end{align*}
By Assumption~\ref{asp:1st_derivative_poly_bound}, using Claim~\ref{cla:chi_squared_tail} and $B \geq d \cdot \poly \log(d)$ we have
\begin{align*}
    \|\E[ C_i(x_j) ] - m_{4,i} (\ov w_i^{\top}  \alpha)^2 \ov w_i \ov w_i^{\top}\| \lesssim &~ \E[\mathbbm{1}_{\|x_j\| \geq B} |w_i|^{p+1}\cdot |x_j |^{p+5}]\\
    \lesssim &~ (\|w_i\| d)^{p+5} \cdot e^{-\Omega(d \log d)}\\
    \lesssim &~ \epsilon.
\end{align*}
Also, $\frac{1}{2}|m_{4,i}|(\ov w_i^\top \alpha)^2 \leq \|\E[ C_i(x_j) ]\| \leq 2|m_{4,i}|(\ov w_i^\top \alpha)^2$. 

For any constant $t  \in (0,1)$, we have with probability  $1- t/4$,
\begin{align*}
\|C_i(x_j) \|  
\lesssim & (| w_i^{\top } x_j |^{p+1} + |\phi(0)|+ |\xi_j|) \cdot (( x_j^\top  \alpha)^2\| x_j \|^2 + 1 +\| x_j\|^2 + ( \alpha^\top  x_j )^2) \\
\lesssim &~ (  \| w_i\|^{p+1} + |\phi(0)| + \vartheta) \cdot d \cdot\poly(\log (d/t))
\end{align*}
where the first step comes from Assumption~\ref{asp:1st_derivative_poly_bound} and the second step comes from Claim~\ref{cla:gaussian_inner_prod_bound} and Claim~\ref{cla:gaussian_norm_bound}.

Using Claim~\ref{cla:three_gaussian}, we have
\begin{align*}
\left\|\E[ C_i(x_j)^2] \right\| \lesssim & ~ \left( \E [ (\phi( w_i^{\top} x_j) + \xi_j)^4] \right)^{1/2} \left( \E [ ( x_j^\top  \alpha)^8 ] \right)^{1/2} \left( \E[\| x_j\|^4] \right)^{1/2} \\
\lesssim & ~ ( \| w_i\|^{p+1}+|\phi(0)| + \vartheta)^2 d.
\end{align*}

Furthermore we have,
\begin{align*}
 \max_{\| a\|=1} \left(\E \left[ ( a^\top C_i(x_j)  a)^2 \right] \right)^{1/2} \lesssim \left( \E \left[ (\phi( w_i^{\top} x_j)+ \xi_j)^4 \right] \right)^{1/4} \lesssim \| w_i\|^{p+1}+|\phi(0)| + \vartheta.
\end{align*}

Then by Claim~\ref{cla:modified_bernstein_non_zero}, with probability at least $1-t$,
\begin{align*}
&~ \left\| m_{4,i} (\ov w_i^{\top}  \alpha)^2 \ov w_i \ov w_i^{\top}- \frac{1}{|S|}\sum_{x_j\in S} C_i(x_j) \right\|\\
\leq &~ \left\| m_{4,i} (\ov w_i^{\top}  \alpha)^2 \ov w_i \ov w_i^{\top}- \E[C_i(x_j)] \right\| + \left\| \E[C_i(x_j)]- \frac{1}{|S|}\sum_{x_j\in S} C_i(x_j) \right\|\\
\lesssim &~ \epsilon |m_{4,i}| (\ov w_i^{\top}  \alpha)^2 + \epsilon.
\end{align*}

Summing up all components $i \in [k]$, we proved Eq.~\eqref{eq:M4(I,I,a,a)}. Eq.~\eqref{eq:M2(I,I)} and Eq.~\eqref{eq:M3(I,I,a)} can be shown similarly.

\end{proof}

\begin{lemma}\label{lem:error_R3}
Let $V \in \R^{d \times k}$ be an orthogonal matrix.
Let $\hat R_3$ be computed in Line~\ref{lin:R_3} of Algorithm~\ref{alg:mom} and $R_3 = P_3(V,V,V)$. Suppose 
$$m_0=\min_{i\in[k]} \{|m_{j_3,i} |^2 (\ov w_i^{\top}  \alpha)^{2(j_3-3)} \}$$ and
\begin{align*}
|S| \gtrsim d \cdot  \poly(\kappa, \vartheta, \log( d/t))  / (\epsilon^2 m_0)
\end{align*}
then with probability at least $1-t$,
\begin{align*}
    \|R_3 - \hat R_3\| \lesssim \epsilon \sum_{i = 1}^k|v_i m_{j_3,i}(\Bar{w_i}^\top \alpha)^{j_3-3}| + \epsilon.
\end{align*}
\end{lemma}

\begin{proof}
From the definition of $R_3$, it suffices to bound $\|M_3(V,V,V) - \hat M_3(V,V,V)\|$ and $\|M_4(V,V,V,\alpha) - \hat M_4(V,V,V,\alpha)\|$. The proof is similar to the previous one.

Specifically, we show that with probability at least $1-t/4$,
\begin{align}\label{eq:M3(V,V,V)}
    \|M_3(V,V,V) - \wh{M}_3(V,V,V)\| \lesssim  \epsilon \sum_{i=1}^k  |v_i m_{3,i}| + \epsilon.
\end{align}

\begin{align}\label{eq:M4(V,V,V,a)}
    \|M_4(V,V,V, \alpha) - \wh{M}_4(V,V,V, \alpha)\| \lesssim \epsilon \sum_{i=1}^k  |v_i m_{4,i} (\ov w_i^{\top}  \alpha)| + \epsilon.
\end{align}

Recall that for sample $(x_j,y_j) \in S$, $y_j = \sum_{i=1}^k v_i \sigma(w_i^\top x_j) + \xi_j$ where $\xi_j$ is independent of $x_j$. 
Consider each component $i \in [k]$. 
Define $T_i(x_j), S_i(x_j) \in \R^{k \times k \times k}$:
\begin{align*}
    T_i(x_j) = &~ (\sigma(w_i^\top x_j) + \xi_j)\\
    &~ \cdot \left(x_i^\top \alpha \cdot v(x)^{\otimes 3} - (V^\top \alpha) \Tilde{\otimes}(v(x) \otimes v(x)) - \alpha^\top x \cdot v(x) \Tilde{\otimes} I + (V^\top \alpha)\Tilde{\otimes} I \right),\\
    S_i(x_j) = &~ \mathbbm{1}({\|x_j\| \leq B}) \cdot T_i(x_j) 
\end{align*}
where $v(x) = V^\top x$.
Flatten $T_i(x_j)$ along the first dimension to obtain $B_i(x_j) \in \R^{k \times k^2}$, flatten $S_i(x_j)$ along the first dimension to obtain $C_i(x_j) \in \R^{k \times k^2}$.

From Claim~\ref{cla:P2_P3}, $\E[ B_i(x_j) ] =   m_{4,i} ( \alpha^\top \ov w_i)(V^\top \ov w_i) \text{vec}((V^\top \ov w_i)(V^\top \ov w_i)^\top)^\top$. Therefore we have,
\begin{align*}
\left\| \E[ B_i(x) ] \right\| = |m_{4,i} ( \alpha^\top \ov w_i)| \cdot \|V^\top \ov w_i\|^{3} .
\end{align*}

We calculate
\begin{align*}
\|\E_{\xi_j}[B_i(x_j)] \|  
\lesssim & (| w_i^{\top } x_j |^{p+1} + |\phi(0)|) \cdot (( x_j^\top  \alpha)^2\| V^\top x_j \|^3\\
&~ + 3\|V^\top x_j\|^3 + 3|x_j^\top \alpha|\|V^\top x_j\|\sqrt{k} + 3\|V^\top \alpha\|\sqrt{k}) \\
\lesssim &~ \sqrt{k} \cdot \|w_i\|^{p+1} \|x_j\|^{p+6}.
\end{align*}
By Assumption~\ref{asp:1st_derivative_poly_bound}, using Claim~\ref{cla:chi_squared_tail} and $B \geq d \cdot \poly \log(d)$,
\begin{align*}
    &~ \|\E[C_i(x_j)] - m_{4,i} ( \alpha^\top \ov w_i)(V^\top \ov w_i) \text{vec}((V^\top \ov w_i)(V^\top \ov w_i)^\top)^\top \|\\ \lesssim &~  \E[\mathbbm{1}_{\|x_j\| \leq B} \sqrt{k} \|w_i\|^{p+1} \|x_j\|^{p+6}]\\ 
    \leq &~ \epsilon.
\end{align*}

For any constant $t \in (0,1)$, we have with probability  $1- t$,
\begin{align*}
\|C_i(x_j) \|  
\lesssim & (| w_i^{\top } x_j |^{p+1} + |\phi(0)|+ |\xi_j|) \cdot (( x_j^\top  \alpha)^2\| V^\top x_j \|^3\\
&~ + 3\|V^\top x_j\|^3 + 3|x_j^\top \alpha|\|V^\top x_j\|\sqrt{k} + 3\|V^\top \alpha\|\sqrt{k}) \\
\lesssim &~ (  \| w_i\|^{p+1} + |\phi(0)| + \vartheta) k^{3/2} \poly(\log (d/t))
\end{align*}
where the first step comes from Assumption~\ref{asp:1st_derivative_poly_bound} and the second step comes from Claim~\ref{cla:gaussian_inner_prod_bound} and Claim~\ref{cla:gaussian_norm_bound}.

Using Claim~\ref{cla:three_gaussian}, we have
\begin{align*}
\left\|\E [C_i(x_j) C_i(x_j)^\top ] \right\| \lesssim & ~ \left( \E \left[ (\phi( w_i^{\top} x_j) + \xi_j)^4 \right] \right)^{1/2}  \left( \E \left[ ( \alpha^\top x_j)^4 \right] \right)^{1/2} \left( \E \left[\|V^\top  x_j\|^6\right] \right)^{1/2} \\
\lesssim & ~(\| w_i\|^{p+1}+|\phi(0)| + \vartheta)^2 k^{3/2}.
\end{align*}
and
\begin{align*}
& ~ \left\| \E[ C_i(x_j)^\top C_i(x_j) ] \right\|  \\
\lesssim & ~ \left( \E[ (\phi( w_i^{\top} x_j) + \xi_j)^4]\right)^{1/2} \left( \E [ ( \alpha^\top x_j)^4 ] \right)^{1/2} \left( \E[\|V^\top  x_j\|^4] \right)^{1/2}  \\
& ~ \cdot \left( \max_{\|A\|_F=1}\E \left[ \langle A,(V^\top  x_j)(V^\top x_j)^\top \rangle^4 \right] \right)^{1/2} \\
\lesssim & ~ ( \| w_i\|^{p+1}+|\phi(0)| + \vartheta)^2 k^2.
\end{align*}

Furthermore we have,
\begin{align*}
&~ \max_{\| a\|=\| b\| = 1} \left(\E\left[ ( a^\top C_i(x_j)  b)^2 \right] \right)^{1/2} \\
\lesssim & ~\left( \E [ (\phi( w_i^{\top} x_j) + \xi_j)^4 ] \right)^{1/4} \left( \E \left[( \alpha^\top x_j)^4 \right]\right)^{1/4} \max_{\| a\|=1} \left( \E \left[( a^\top V^\top  x_j)^4 \right]\right)^{1/2} \\
& ~ \cdot \max_{\|A\|_F=1} \left( \E  \left[ \langle A,(V^\top  x_j)(V^\top  x_j)^\top \rangle^4 \right] \right)^{1/2} \\
\lesssim & ~ ( \| w_i\|^{p+1}+|\phi(0)| + \vartheta)k.
\end{align*}

Then by Claim~\ref{cla:modified_bernstein_non_zero}, with probability at least $1-t$,
\begin{align*}
&~ \left\|m_{4,i} ( \alpha^\top \ov w_i)(V^\top \ov w_i) \text{vec}((V^\top \ov w_i)(V^\top \ov w_i)^\top)^\top - \frac{1}{|S|}\sum_{x_j\in S} C_i(x_j) \right\|\\
\leq &~ \left\| m_{4,i} ( \alpha^\top \ov w_i)(V^\top \ov w_i) \text{vec}((V^\top \ov w_i)(V^\top \ov w_i)^\top)^\top- \E[C_i(x_j)] \right\|\\
&~ + \left\| \E[C_i(x_j)]- \frac{1}{|S|}\sum_{x_j\in S} C_i(x_j) \right\|\\
\lesssim &~ \epsilon |v_i m_{4,i} (\ov w_i^{\top}  \alpha)| + \epsilon.
\end{align*}

Summing up all neurons $i \in [k]$, we proved Eq.~\eqref{eq:M4(V,V,V,a)}. Eq.~\eqref{eq:M3(V,V,V)} can be shown similarly.

\end{proof}

\begin{lemma}\label{lem:error_Q1}
Let $\hat Q_1$ and $\hat Q_2$ be computed in Line~\ref{lin:Q} of Algorithm~\ref{alg:mom}. Let $Q_1$ be defined by Eq.~\ref{eq:def_Q1} and $Q_2$ be defined by Eq.~\ref{eq:def_Q2}. Suppose 
$$m_0=\min_{i\in[k]} \{|m_{j_1,i} |^2 (\ov w_i^{\top}  \alpha)^{2(j_1-1)},|m_{j_2,i} |^2 (\ov w_i^{\top}  \alpha)^{2(j_2-2)} \}$$
and
\begin{align*}
|S| \gtrsim d \cdot  \poly(\kappa, \vartheta, \log (d/t))  / (\epsilon^2 m_0)
\end{align*}
then with probability at least $1-t$,
\begin{align*}
    \|Q_1 - \hat Q_1\| \lesssim &~  \epsilon \sum_{i = 1}^k|v_i m_{j_1,i}(\Bar{w_i}^\top \alpha)^{j_1-1}| + \epsilon,\\
    \|Q_2 - \hat Q_2\| \lesssim &~  \epsilon \sum_{i = 1}^k|v_i m_{j_2,i}(\Bar{w_i}^\top \alpha)^{j_2-2}| + \epsilon.
\end{align*}
\end{lemma}

\begin{proof}
Recall the expression of $Q_1$ and $Q_2$,
\begin{align*}
& Q_1 = M_{l_1}(I,\underbrace{ \alpha,\cdots, \alpha}_{(j_1 - 1)~ \alpha\text{'s}}) = \sum_{i=1}^k v_i c_{j_1} \| w_i\|^{p+1} ( \alpha^\top \ov{ w}_i)^{{j_1-1}} \ov{ w}_i, \\
& Q_2 = M_{j_2}(V,V,\underbrace{ \alpha,\cdots, \alpha}_{(j_2 - 2)~ \alpha\text{'s}}) = \sum_{i=1}^k v_i c_{j_2} \| w_i\|^{p+1} ( \alpha^\top \ov{ w}_i)^{{j_2-2}} (V^\top\ov{ w}_i)(V^\top\ov{ w}_i)^\top.
\end{align*}
The proof is essentially similar to Lemma~\ref{lem:error_P2} and Lemma~\ref{lem:error_R3}.
\end{proof}

We also use the following Lemmata from \cite{kcl15,zsj+17}.

\begin{lemma}[Adapted from Theorem 3 of \cite{kcl15}]\label{lem:tensor_fact}
Given a tensor $\hat T = \sum_{i=1}^k \pi_i u_i^{\otimes 3} + \epsilon R \in \R^{d \times d\times d}$. Assume incoherence $u_i^\top u_j \leq \mu$. Let $L_0 := (\frac{50}{1-\mu^2})^2$ and $L \geq L_0 \log (15d(k-1)/t)^2$. Then there exists an algorithm such that, with probability at least $1-t$, for every $u_i$, the algorithm returns a $\Tilde{u}_i$ such that
\begin{align*}
    \|\Tilde{u}_i - u_i \|_2 \leq O\left( \frac{\sqrt{\|\pi\|_1 \pi_{\max}}}{\pi_{\min}^2} \cdot \frac{\|V^\top\|_2^2}{1-\mu^2} \cdot (1+C(t)) \right)\epsilon + o(\epsilon),
\end{align*}
where $C(t) := \log(kd/t)\sqrt{d/L}$ and $V$ is the inverse of the full-rank extension of $(u_1 \dots u_k)$ with unit-norm columns.
\end{lemma}

\begin{lemma}[Adapted from Lemma E.6 of \cite{zsj+17}]\label{lem:subspace_estimation}
Let $P_2$ be defined as in Definition~\ref{def:P_2_P_3} and $\wh{P}_2$ be its empirical version calculated in Line~\ref{lin:P_2} of Algorithm~\ref{alg:mom}. Let $U \in \mathbb{R}^{d\times k}$ be the orthogonal column span of $W \in \R^{d\times k}$. Assume $\|\wh{P}_2 - P_2\| \leq s_k(P_2)/10$. Let $C$ be a large enough positive number such that $C>2\|P_2\|$.
Then after $T = O(\log(1/\epsilon))$ iterations, the $V \in \R^{d\times k}$ computed in Algorithm~\ref{alg:mom} will satisfy
\begin{align*}
\|UU^\top-VV^\top\| \lesssim  \|\wh{P}_2 - P_2\|  / s_k(P_2) +\epsilon, 
\end{align*}
which implies
\begin{align*}
\|(I-VV^\top)  w_i\| \lesssim (\|\wh{P}_2 - P_2\|  / s_k(P_2)+\epsilon) \| w_i\| . 
\end{align*}
\end{lemma}

\begin{lemma}[Adapted from Lemma E.13 in \cite{zsj+17}]\label{lem:solution_system_1}

Let $U\in \R^{d\times k}$ be the orthogonal column span of $W^*$. Let $V\in \R^{d\times k}$ denote an orthogonal matrix satisfying that $\| VV^\top - UU^\top \| \leq  \wh{\delta}_2 \lesssim 1/ ( \kappa^2 \sqrt{k})$. For each $i\in[k]$, let $\widehat{u}_i$ denote the vector satisfying $\| s_i \wh{u}_i - V^\top \ov{ w}_i\| \leq \wh{\delta}_3\lesssim  1/ ( \kappa^2 \sqrt{k})$.
Let $Q_1$ be defined as in Eq.~\eqref{eq:def_Q1} and $\wh{Q}_1$ be the empirical version of $Q_1$ such that $\|Q_1 - \wh{Q}_1\| \leq \wh{\delta}_4 \|Q_1\| \leq \frac{1}{4}\|Q_1\|$. 

Let $z^* \in \R^k$ and $\wh{z} \in \R^k$ be defined as in Eq.~\eqref{eq:zr_star} and Line~\ref{lin:linear_system}.
Then 
\begin{align*}
| \wh{z}_i - z_i^*| \leq (\kappa^4 k^{3/2} (\wh{\delta}_2+\wh{\delta}_3) + \kappa^2 k^{1/2} \wh{\delta}_4)\|z^*\|_1.
\end{align*}
\end{lemma}

\begin{lemma}[Adapted from Lemma E.14 in \cite{zsj+17}]\label{lem:solution_system_2}
Let $U\in \R^{d\times k}$ be the orthogonal column span of $W^*$ and $V$ be an orthogonal matrix satisfying that $\| VV^\top - UU^\top \| \leq  \wh{\delta}_2 \lesssim 1/(\kappa \sqrt{k})$. For each $i\in[k]$, let $\wh{u}_i$ denote the vector satisfying $\| s_i \wh{u}_i - V^\top \ov{ w}_i^*\| \leq \wh{\delta}_3 \lesssim 1/(\sqrt{k}\kappa^3)$.

Let $Q_2$ be defined as in Eq.~\eqref{eq:def_Q2} and $\wh{Q}_2$ be the empirical version of $Q_2$ such that $\|Q_2 - \wh{Q}_2\|_F \leq \wh{\delta}_4 \|Q_2\|_F \leq  \frac{1}{4}\|Q_2\|_F$.
Let $ r^* \in \R^k$ and $ \wh{r} \in \R^k$ be defined as in Eq.~\eqref{eq:zr_star} and Line~\ref{lin:linear_system}.
Then
\begin{align*}
|\wh{r}_i - r_i^*| \leq  ( k^3\kappa^{8}\wh{\delta}_3 +  \kappa^2 k^{2} \wh{\delta}_4) \| r^*\| .
\end{align*} 
\end{lemma}

Now we are in the position of proving Theorem~\ref{thm:mom_recover_nn}.
\begin{proof}
Consider Algorithm~\ref{alg:mom}. 
First, by Lemma~\ref{lem:subspace_estimation} and Lemma~\ref{lem:error_P2}, we have
\begin{align}\label{eq:wi_minus_siVui_part1}
\| VV^\top \ov w_i -\ov w_i \| \leq & ~ ( \| \wh{P}_2 - P_2 \| / s_k(P_2) + \epsilon)  \notag \\
\leq & ~ ( \poly(k,\kappa) \| \wh{P}_2 - P_2 \| + \epsilon) \notag \\
\leq & ~  \poly(k,\kappa)  \epsilon.
\end{align}

Next, combining Lemma~\ref{lem:tensor_fact} and Lemma~\ref{lem:error_R3}, we have
\begin{align}\label{eq:wi_minus_siVui_part2}
\|V^\top \ov w_i - s_i \wh{u}_i\| \leq  \poly(k,\kappa) \|\wh{R}_3 - R_3\|
\leq  \epsilon \poly(k, \kappa).
\end{align}

It thus follows that
\begin{align}\label{eq:wi}
\|\ov w_i- s_iV \hat {u}_i\| \leq & ~\| VV^\top \ov w_i -\ov w_i \| + \| V V^\top \ov w_i - V s_i \hat {u}_i\| \notag \\
= &~ \| VV^\top \ov w_i -\ov w_i \| + \| V^\top \ov w_i - s_i \hat {u}_i\|\notag \\
\leq &~ \epsilon \poly(k,\kappa),
\end{align}
where the first step applies triangle inequality and the last step uses Eq.~\eqref{eq:wi_minus_siVui_part1} and Eq.~\eqref{eq:wi_minus_siVui_part2}.

We proceed to bound the error in $\wh{r}$ and $\wh{z}$. We have,
\begin{align}\label{eq:r-r*}
|\wh{r}_i - r_i^*| \lesssim &~ \poly(k,\kappa) ( \|Q_2-\wh{Q}_2\| + \|s_i \hat{u}_i - V^\top \ov w_i\|) \cdot \| r^*\|\notag \\
\lesssim &~ \epsilon \poly(k,\kappa) \cdot \|r^*\|
\end{align}
where the first step comes from Lemma~\ref{lem:solution_system_2} and the second step comes from Lemma~\ref{lem:error_Q1} and Eq.~\eqref{eq:wi}.
Furthermore,
\begin{align}\label{eq:z-z*}
| \wh{z}_i - z_i^*| \lesssim &~ \poly(k,\kappa) \cdot (\|Q_1-\wh{Q}_1\|+ \|s_i \hat{u}_i - V^\top \ov w_i\| + \|VV^\top - UU^\top \|) \cdot \|z^*\|_1 \notag \\
\lesssim &~ \epsilon \poly(k,\kappa) \|z^*\|_1,
\end{align}
where the first step comes from Lemma~\ref{lem:solution_system_1} and the second step comes from combining Lemma~\ref{lem:error_Q1}, Lemma~\ref{lem:subspace_estimation}, and Eq.~\eqref{eq:wi}.
Finally, combining Eq.~\eqref{eq:r-r*}, Eq.~\eqref{eq:z-z*} and  Eq.~\eqref{eq:wi}, the output in Line~\ref{lin:wi0} satisfies $\|{\hat w}_i - w_i \|_F \leq \epsilon \poly(k,\kappa)\cdot \|w_i\|_F$. Since $v_i$ are discrete values, they are exactly recovered.

\end{proof}

%% file: appendix/gd+mom.tex
\subsubsection{Exact recovery of neural networks from noiseless data}\label{sec:gd+mom}

In this section we prove Theorem~\ref{thm:exact_nn}. 
Similar to Appendix~\ref{sec:mom}, denote $W = [w_1,\cdots,w_k]^\top$ where $w_i \in \R^d$ and $\hat W = [\hat w_1,\cdots,\hat w_k]^\top$. We use $\cal D$ to denote the distribution of $x \sim \mathcal{N}(0,I_d)$ and $y = \langle v, \sigma(W x) \rangle$. 
We define the empirical loss for explored features and the population loss as follows,
\begin{align}
    \label{eq:empirical_loss}
L_n(\hat W) = &~ \frac{1}{2n}\sum_{(x,y) \in S^{(1)}} \left( \sum_{i=1}^k v_i \sigma( \hat w_{i}^\top  x_i) - y_i \right) ^2,\\
\label{eq:population_loss}
L(\hat W) = &~ \frac{1}{2} {\mathbb{E}_{\cal D}} \left[\left( \sum_{i=1}^k v_i \sigma( \hat w_{i}^\top  x) - y \right) ^2\right].
\end{align}

\begin{algorithm*}[h]\caption{Using method of moments and gradient descent to recover neural network parameters}\label{alg:mom_gd}
\begin{algorithmic}[1]
\Procedure{$\textsc{NeuralNetRecovery}$}{$S = \{(x_i,y_i): i \in [n]\}$}
\State Let $S^{(1)} \leftarrow \{(x,y) \in S: \|x\| \leq B\}$
\State Compute $(v,\hat W) \leftarrow \textsc{NeuralNetNoisyRecovery}(S^{(1)})$
\State Find $W^{(1)}$ as the global minimum of $L_n(\cdot)$ by gradient descent initialized at $\hat W$, where $L_n(\cdot)$ is defined in Eq.~\eqref{eq:empirical_loss}.
\State {\bf Return} $(v,W^{(1)})$
\EndProcedure

\end{algorithmic}
\end{algorithm*}

\begin{definition}
	Let $s_i$ be the $i$-th singular value of $W$, $\lambda:=\prod_{i=1}^{k}(s_i/s_k)$. Let $\tau = (3s_1/2)^{4p}/\min_{z\in [s_k/2,3s_1/2]} \{\rho^2(z) \}$.
\end{definition}

We use the follow results adapted from \cite{zsj+17}. The only difference is that the rewards are potentially truncated if $\|x \| \geq B$, and due to $B = d \cdot \poly \log (d)$ we can bound its difference between standard Gaussian in the same way as Appendix~\ref{sec:mom}.
\begin{lemma}[Concentration, adapted from Lemma D.11 in \cite{zsj+17}]\label{lem:concentration_hessian}
Let samples size $n\geq \epsilon^{-2}d\tau \poly(\log ( d/t))$, then with probability at least $1-t$,
\begin{align*}
\|\nabla^2L_n(W)- \nabla^2 L(W)\| \lesssim ks_1^{2p} \epsilon  + \poly(B_W, d) e^{-\Omega(d)}.
\end{align*}
\end{lemma}

\begin{lemma}[Adapted from Lemma D.16 in \cite{zsj+17}]\label{lem:empirical_lipschitz_hessian}
Assume activation $\sigma(\cdot)$ satisfies Assumption \ref{asp:hessian_activation} and Assumption~\ref{asp:rho>0}. Then for any $t  \in (0,1)$, if
$n \geq d \cdot \poly(\log (d/t))$,
with probability at least $1 -t$, for any $\hat W$ (which is not necessarily to be independent of samples) satisfying $\|W - \hat{W}\| \leq s_k/4$, we have
\begin{align*}
\|\nabla^2 L_n(\hat W)-  \nabla^2 L_n(W)\|\leq k s_1^p \|W - \hat{W}\| d^{(p+1)/2}. 
\end{align*}
\end{lemma}

Now we prove Theorem~\ref{thm:exact_nn}.
\begin{proof}
The exact recovery consists of first finding (exact) $v$ and (approximate) $\hat W$ close enough to $W$ by tensor method (Appendix~\ref{sec:mom}), and then minimizing the empirical loss $L_n(\cdot)$. We will prove that $L_n(\cdot)$ is locally strongly convex, thus we find the precise $W$.

From Lemma D.3 from \cite{zsj+17} we know:
\begin{equation}
\Omega(\rho(s_k)/\lambda) I \preceq \nabla^2 L(W)\preceq O(k  s_1^{2p}) I.  
\end{equation}

Combining Lemma~\ref{lem:concentration_hessian}, $d \geq \log (B_W / \lambda)$, and $n\geq \frac{k^2\lambda^2s_1^{4p}}{\rho^2(s_k)}d\tau \poly(\log (d/ t)) $, we know $\nabla^2 L_n(W)$ must be positive definite.

Next we uniformly bound Lipschitzness of $\nabla^2 L_n$. 
From Lemma~\ref{lem:empirical_lipschitz_hessian} there exists a universal constant $c$, such that for all  $\hat W$ that satisfies $ \|W - \hat{W}\|\leq c ks_1^{2p} / (ks_1^p d^{(p+1)/2}) = c s_1^p d^{-(p+1)/2} $, $\nabla L_n^2(\hat W)\gtrsim k s_1^{2p}$ holds uniformly. So there is a unique miminizer of $L_n$ in this region. 

Notice $L_n(W) = 0$, therefore we can find $W$ by directly minimizing the empirical loss as long as we find any $\hat W$ in this region. This can be achieved by tensor method in Appendix~\ref{sec:mom}. We thus complete the proof.  
\end{proof}

%% file: appendix/ub_det_mdp.tex
\section{Omitted Proofs in Section~\ref{sec:det:mdp}}

For the proofs of Theorem~\ref{thm:ag}, Example~\ref{ex:rank-k}, and Example~\ref{example:qUx}, we refer the readers to \cite{huang2021optimal}.

\begin{lemma} \label{lem:ag:sample}
Consider the polynomial family $\F_{\gV}$ of dimension $D$. Assume that $n>2D$. For any $E \in \mathbb{R}^d$ that is of positive measure, by sampling $n$ samples $\{x_i\}$ i.i.d. from $\PP_{x \in \mathcal{N}(0,I_d)} (\cdot| x \in E)$ and observing the noiseless feedbacks $y_i = f^*(x_i)$, one can almost surely uniquely determine the $f^*$ by solving the system of equations $y_i =f(x_i)$, $i=1,\dots,n$, for $f \in \F_{\gV}$.
\end{lemma}
\begin{proof}
By Theorem~\ref{thm:ag}, there exists a set $N \in \mathbb{R}^d \times \ldots \mathbb{R}^d$ of Lebesgue measure zero, such that if $(x_1,\cdots,x_n) \notin N$, one can uniquely determine the $f^*$ by the observations on the $n$ samples. Therefore, we only need to show that with probability 1, the sampling procedure returns $(x_1,\dots,x_n) \notin N$. This is because
\begin{align*}
    \PP(x_1,\dots,x_n \in N) &= \PP_{x_i \in \gN(0,I_d)}(  (x_1,\dots,x_n) \in N  \mid  x_1, \dots, x_n \in E) \\
    &= \frac{ \PP_{x_i \in \gN(0,I_d)}(  (x_1,\dots,x_n) \in N \cap (E \times \dots \times  E)) }{\PP_{x_i \in \gN(0,I_d)}(  (x_1,\dots,x_n) \in (E \times \dots \times  E)) }  \\
    &= \frac{0}{[ \PP_{x_1 \in \gN(0,I_d)}(  x_1 \in  E)  ]^n} \\
    &=0.
\end{align*}
\end{proof}

By Lemma~\ref{lem:ag:sample} above, it is not hard to see that Algorithms~\ref{alg:querypoly} and~\ref{alg:onlinerlpoly} work.

%% file: appendix/lb_det_mdp.tex
\section{Omitted Constructions and Proofs in Subsection~\ref{sec:det_lb}}

\paragraph{Construction of the Reward Functions} The following construction of the polynomial hard case is adopted from \cite{huang2021optimal}.

Let $d$ be the dimension of the feature space. Let $e_i$ denotes the $i$-th standard orthonormal basis of $\R^{d}$, i.e., $e_i$ has only one $1$ at the $i$-th entry and $0$'s for other entries. Let $p$ denote the highest order of the polynomial. We assume $d \gg p$. We use $\Lambda$ to denote a subset of the $p$-th multi-indices
\[
    \Lambda = \{ (\alpha_1,\dots,\alpha_p) | 1 \leq \alpha_1 \leq \dots \leq \alpha_p \leq d \}.
\]
For an $\alpha = (\alpha_1,\dots,\alpha_p) \in \Lambda$, denote $M_\alpha = e_{\alpha_1} \otimes \dots \otimes e_{\alpha_p}$, $x_\alpha =  e_{\alpha_1} + \dots + e_{\alpha_p}$. 

The model space $\cM$ is a subset of rank-1 $p$-th order tensors, which is defined as $\cM =   \{ M_\alpha | \alpha \in \Lambda \}$. We define two subsets of feature space $\cF_0$ and $\cF$ as $\cF_0 = \{ x_\alpha | \alpha \in \Lambda \}, \ \cF = \mathrm{conv}(\cF_0)$.
For $M_\alpha \in \cM$, $x \in \cF$, define $r(M_\alpha, x)$ as $ r(M_\alpha, x) = \langle M_\alpha, x^{\otimes p} \rangle =  \prod_{i=1}^p \langle e_{\alpha_i}, x\rangle$.
We assume that for each level $h$, there is a $M^{(h)} = M_{\alpha^{(h)}} \in \cM$, and the noiseless reward is $r_h(s,a) =  r(M^{(h)}, \phi_h(s,a))$.

We have the following properties.
\begin{proposition} [\cite{huang2021optimal}]
\label{prop:lb:delta_f}
For $M_\alpha \in \cM$ and $x_{\alpha'} \in \cF_0$, we have
\[
    r(M_\alpha, x_{\alpha'}) = \indict_{\{\alpha = \alpha' \}}.
\]
\end{proposition}

\begin{proposition}
\label{prop:lb:max_reward}
For $M_\alpha \in \cM$ , we have
\[
    \max_{x \in \cF} r(M_\alpha, x) =1.
\]
\end{proposition}
\begin{proof}[proof of Proposition~\ref{prop:lb:max_reward}]
For all $x \in \cF $,  since $\cF = \mathrm{conv}(\cF_0)$, we can write 
\[
    x = \sum_{\alpha \in \Lambda} p_\alpha (e_{\alpha_1} + \dots + e_{\alpha_p}),
\]
where $\sum_{\alpha \in \Lambda } p_{\alpha} = 1$ and $p_{\alpha} \geq 0$. Therefore, 
\[
    r (M_{\alpha'}, x) =  \prod_{i=1}^p  \langle e_{\alpha'_i},  x  \rangle  .
\]
Plug in the expression of $x$, we have
\begin{align*}
    \langle e_{\alpha'_i}, x \rangle &= \sum_\alpha p_\alpha  \langle  e_{\alpha'_i} ,e_{\alpha_1} + \dots + e_{\alpha_p} \rangle  \\
    &= \sum_\alpha p_\alpha  \indict_{\{\alpha'_i \in \alpha \}} \\
    &\leq  \sum_\alpha p_\alpha = 1.
\end{align*}
Therefore, 
\begin{align*}
 r(M_{\alpha'}, x) &=  \prod_{i=1}^p  \langle e_{\alpha'_i},  x  \rangle  \\
   &= \Big( \sum_\alpha p_\alpha  \indict_{\{e_{\alpha'_1}  \in \alpha \}} \Big) \cdots \Big( \sum_\alpha p_\alpha  \indict_{\{e_{\alpha'_p}  \in \alpha \}} \Big) \\
    &\leq 1.
\end{align*}
Finally, since $r(M_{\alpha'}, x_{\alpha'}) = 1$, we have $\max_{x \in \cF} r(M_\alpha, x) =1$.
\end{proof}

\paragraph{MDP constructions}

Consider a family of MDPs with only two states $\gS = \{S_{\text{good}}, S_{\text{bad}}\}$. The action set $\gA$ is set to be $\cF$. Let $f$ be a mapping from $\cF$ to $\cF_0$ such that $f$ is identity when restricted to $\cF_0$. For all level $h \in [H]$, we define the feature map $\phi_h: S \times A \to \cF$ to be 
\[
    \phi_h(s,a) = \left\{ 
    \begin{array}{ll}
        a & \text{ if } s=S_{\text{good}},\\
        f(a) & \text{ if } s=S_{\text{bad}}.
    \end{array}
    \right.
\] 
Given an unknown sequence of indices $\alpha^{(1)}, \dots, \alpha^{(H)}$, the reward function at level $h$ is $r_h(s, a) = r(M_{\alpha^{(h)}}, \phi_h(s,a))$. Specifically, we have 
\[
    r_h(S_{\text{good}}, a) = r(M_{\alpha^{(h)}}, a),  \ r_h(S_{\text{bad}}, a) = r(M_{\alpha^{(h)}}, f(a)).
\]
The transition $P_h$ is constructed as 
\[
    P_h(S_{\text{bad}}   | s, a) = 1 \text{ for all }  s \in \gS,  a \in \gA.
\]
This construction means it is impossible for the online scenarios to reach the good state for $h>1$.

The next proposition shows that $Q^*_h$ is polynomial realizable and falls into the case of Example~\ref{example:qUx}.
\begin{proposition} \label{prop:lb:poly}
We have for all $h \in [H]$ and $s \in \gS$, $a \in \gA$, $V^*_h(s) = H-h+1$ and $Q^*_h(s, a) = r_h(s, a) + H - h + 1 $. Furthermore, $Q^*_h(s,a)$, viewed as the function of $\phi_h(s,a)$, is a polynomial of the form $q_h(U_h \phi_h(s,a))$ for some degree-$p$ polynomial $q_h$ and $U_h \in \mathbb{R}^{p \times d}$.
\end{proposition}
\begin{proof}[proof of Proposition~\ref{prop:lb:poly}]
First notice that by Proposition~\ref{prop:lb:max_reward}, for all $h \in [H]$ and $s \in \gS$, we have
\[  
    \max_{a \in \gA} r_h(s,a) = 1.
\]
Therefore, by induction, suppose we have proved for all $s'$, $V^*_{h+1}(s') = H-h$, then we have
\begin{align*}
    V^*_h(s) &= \max_{a \in  \gA} Q^*_h(s,a) \\
    &= \max_{a \in  \gA} \{ r_h(s,a) + \E_{s' \sim P_h(\cdot|s,a)}  [V^*_{h+1}(s')] \}  \\
    &= 1 + H-h.
\end{align*}
Then we have $Q^*_h(s, a) = r_h(s, a) + H - h + 1$.

Furthermore, we have 
\begin{align*}
    Q^*_h(s, a) &= r_h(s, a) + H - h + 1 \\
            &= r(M_{\alpha^{(h)}}, \phi_h(s,a)) + H - h + 1\\
            &= \prod_{i=1}^p \langle e_{\alpha_i^{(h)}}, \phi_h(s,a)\rangle + H - h + 1 \\
            &= q_h(U_h \phi_h(s,a)),
\end{align*}    
where $q_h(x_1,\dots, x_p) = x_1  x_2 \cdots x_p + (H-h+1)$ and $U_h \in \mathbb{R}^{p \times d}$ is a matrix with $e_{\alpha_i^{(h)}}$ as the $i$-th row. 
\end{proof}

\begin{theorem}\label{thm:lb:ol}
Under the online RL setting, any algorithm needs to play at least $({d \choose p}-1) = \Omega(d^p)$ episodes to identify $\alpha^{(2)}, \dots, \alpha^{(H)}$ and thus to identify the optimal policy.
\end{theorem}
\begin{proof}[proof of Theorem~\ref{thm:lb:ol}]
Under the online RL setting, any algorithm enters and remains in $S_{\text{bad}}$ for $h>1$. When $s_h = S_{\text{bad}}$, no matter what $a_h$ the algorithm chooses, we have $\phi_h(s_h, a_h) = f(a_h) \in \cF_0$. Notice that for any $M_{\alpha^{(h)}} \in \cM $ and any $x_{\alpha} \in \cF_0$, we have $r(M_{\alpha^{(h)}}, x_{\alpha}) = \indict_{\{\alpha = \alpha^{(h)} \}}$ as Proposition~\ref{prop:lb:delta_f} suggests. Hence, we need to play $({d \choose p}-1)$ times at level $h$ in the worst case to find out $\alpha^{(h)}$. The argument holds for all $h = 2,3,\dots, H$.
\end{proof}

\begin{theorem}\label{thm:lb:ag}
Under the generative model setting, by querying $2d(p+1)^p H = O(dH)$ samples, we can almost surely identify $\alpha^{(1)}$ ,$\alpha^{(2)}, \dots, \alpha^{(H)}$ and thus identify the optimal policy.
\end{theorem}
\begin{proof}[proof of Theorem~\ref{thm:lb:ag}]
By Proposition~\ref{prop:lb:poly}, we know that $Q^*_h(s,a)$, viewed as the function of $\phi_h(s,a)$, falls into the case of Example~\ref{example:qUx} with $k=p$. 

Next, notice that for all $h\in[H]$, $\{\phi_h(s,a) \mid s \in \gS, a \in \gA \} = \gF$. Although $\gF$ is not of positive measure, we can actually know the value of $Q^*_h$ when $\phi_h(s,a)$ is in $\mathrm{conv}(\gF, \mathbf{0})$ since the reward is $p$-homogenous. Specifically, for every feature of the form $ c \cdot\phi_h(s,a)$, where $0\leq c \leq 1$ and $ \phi_h(s,a) \in \gF$, the reward is $c^p$ times the reward of $(s,a)$. Therefore, to get the reward at $c \cdot \phi_h(s,a)$, we only need to query the generative model at $(s,a)$ of level $h$, and then multiply the reward by $c^p$.

Notice that $\mathrm{conv}(\gF, \mathbf{0})$ is of positive Lebesgue measure. By Theorem~\ref{thm:gen-rl}, we know that only $2d(p+1)^pH = O(dH)$ samples are needed to determine the optimal policy almost surely.

\end{proof}

%% file: appendix/technical_claims.tex
\section{Technical claims}

\begin{claim}\label{cla:chi_squared_tail}
Let $\chi^2(d)$ denote $\chi^2$-distribution with degree of freedom $d$. For any $t > 0$ we have,
\begin{align*}
    \Pr_{z \sim \chi^2(d)} (z \geq d + 2t + 2\sqrt{dt}) \leq e^{-t}
\end{align*}
\end{claim}

We use the following facts from \cite{zsj+17}.

\begin{claim}\label{cla:gaussian_inner_prod_bound}
Given a fixed vector $z \in \R^d$, for any $C \geq 1$ and $n\geq 1$, we have
\begin{align*}
\underset{ x\sim {\cal N}(0,I_d) }{ \Pr } [  | \langle x , z \rangle |^2 \leq 5C \| z\|^2 \log n ]  \geq 1-1/(nd^C).
\end{align*}
\end{claim}

\begin{claim}\label{cla:gaussian_norm_bound}
For any $C\geq 1$ and $n\geq 1$, we have
\begin{align*}
\underset{ x\sim {\cal N}(0,I_d) }{ \Pr } [  \| x \|^2 \leq 5C d \log n ]  \geq 1- 1/(nd^C).
\end{align*}
\end{claim}

\begin{claim}\label{cla:three_gaussian}
Let $a,b,c \geq 0$ be three constants, let $u,v,w\in \mathbb{R}^d$ be three vectors, we have
\begin{align*}
\underset{x\sim {\cal N}(0,I_d)}{\E} \left[ |u^\top x|^a |v^\top x|^b |w^\top x|^c \right] \approx \|u\|^a \| v\|^b \| w\|^c.
\end{align*}
\end{claim}

\begin{claim}\label{cla:moments}
Let $M_j,j \in [4]$ be defined in Definition~\ref{def:moments}. For each $j\in [4]$, $ M_j = \sum_{i=1}^k v_i m_{j,i} \ov{w}_i^{\otimes j}$.
\end{claim}

\begin{claim}
\label{cla:modified_bernstein_non_zero}
Let ${\cal B}$ denote a distribution over $\mathbb{R}^{d_1 \times d_2}$. Let $d = d_1 +d_2$. Let $B_1, B_2, \cdots B_n$ be i.i.d. random matrices sampled from ${\cal B}$. Let $\overline{B} = \mathbb{E}_{B\sim {\cal B}} [B]$ and $\wh{B}  = \frac{1}{n} \sum_{i=1}^n B_i$. For parameters $m\geq 0, \gamma \in (0,1),\nu >0 ,L>0$, if the distribution ${\cal B}$ satisfies the following four properties,
\begin{align*}
\mathrm{({1})} \quad & \quad \underset{B \sim {\cal B}}{\Pr}\left[ \left\| B \right\| \leq  m \right] \geq 1 - \gamma; \\ 
\mathrm{({2})} \quad & \quad \left\| \underset{B \sim {\cal B}}{\mathbb{E}}[B]  \right\| >0; \\
\mathrm{({3})} \quad & \quad \max \left( \left\| \underset{B \sim {\cal B}}{\mathbb{E}} [ B B^\top ] \right\|, \left\| \underset{B \sim {\cal B}}{\mathbb{E}} [ B^\top B ] \right\| \right) \leq \nu ;\\ 
\mathrm{({4})} \quad & \quad \max_{\| a\|=\| b\|=1} \left( \underset{B \sim {\cal B}}{\mathbb{E}} \left[ \left( a^\top B  b \right)^2 \right]  \right)^{1/2} \leq L.
\end{align*}

Then we have for any $0<\epsilon <1$ and $t\geq 1$, if
\begin{align*}
n \geq  ( 18 t \log d  ) \cdot ( \nu + \| \ov{B} \|^2+ m \| \ov{B} \| \epsilon )  / ( \epsilon^2 \| \ov{B} \|^2 ) \quad \text{~and~} \quad \gamma \leq (\epsilon \| \ov{B} \| /(2L) )^2
\end{align*} 
with probability  at least $1-1/d^{2t} - n\gamma$,
\begin{equation*}
\| \wh{B} - \ov{B} \| \leq \epsilon \| \ov{B} \|.
\end{equation*}
\end{claim}

\begin{claim}\label{cla:P2_P3}
Let $P_2$ and $P_3$ be defined in Definition~\ref{def:P_2_P_3}. Then 
$$P_2 = \sum_{i=1}^k v_i m_{j_2,i} ( \alpha^\top \ov{ w}_i)^{{j_2-2}} \ov{ w}_i^{\otimes 2}$$
and
$$P_3=\sum_{i=1}^k v_i m_{j_3,i}( \alpha^\top \ov{ w}_i)^{{j_3-3}} \ov{ w}_i^{\otimes 3}.$$
\end{claim}